\documentclass{article}

\PassOptionsToPackage{numbers,sort&compress}{natbib}
\usepackage[preprint]{neurips_2026}

\usepackage[utf8]{inputenc} % allow utf-8 input
\usepackage[T1]{fontenc}    % use 8-bit T1 fonts
\usepackage{hyperref}       % hyperlinks
\usepackage{url}            % simple URL typesetting
\usepackage{booktabs}       % professional-quality tables
\usepackage{amsfonts}       % blackboard math symbols
\usepackage{nicefrac}       % compact symbols for 1/2, etc.
\usepackage{microtype}      % microtypography
\usepackage[table]{xcolor}   % colors; [table] enables \cellcolor in tabular

\usepackage{graphicx}
\usepackage{subcaption}
\usepackage{wrapfig}
\usepackage{tabularx}
\usepackage{array}
\usepackage{booktabs}
\usepackage{multirow}
\usepackage{makecell}
\usepackage{adjustbox}
\usepackage[ruled]{algorithm2e}
\usepackage{amsmath,amsfonts,amssymb}
\usepackage{amsthm}
\usepackage{mathrsfs}

\usepackage{textcomp}
\usepackage{stfloats}
\usepackage{verbatim}
\usepackage{float}
\usepackage{pifont}
\usepackage{enumitem}
\usepackage[normalem]{ulem}
\usepackage{manyfoot}
\usepackage{listings}
\usepackage{lipsum}
\usepackage{xspace}
\usepackage[most]{tcolorbox}
\usepackage{type1cm}
\usepackage{titletoc}

\usepackage{pifont}

\newtheorem{theorem}{Theorem}
\tcbset{
  aibox/.style={
    width=\linewidth,
    top=8pt,
    bottom=4pt,
    colback=blue!6!white,
    colframe=black,
    colbacktitle=black,
    enhanced,
    center,
    attach boxed title to top left={yshift=-0.1in,xshift=0.15in},
    boxed title style={boxrule=0pt,colframe=white,},
  }
}
\newtcolorbox{AIbox}[2][]{aibox,title=#2,#1}
\title{TINS: Test-time ID-prototype-separated Negative Semantics Learning for OOD Detection}

\author{%
  Yifeng Yang\textsuperscript{\rm1},
  Jubo Feng\textsuperscript{\rm1},
   Jing Xu\textsuperscript{\rm4},
  Xinbing Wang\textsuperscript{\rm1},
  Qinying Gu\textsuperscript{\rm2},
  Nanyang Ye\textsuperscript{\rm1,2,3} \thanks{Nanyang Ye is the corresponding author.} \\
  \textsuperscript{\rm1} Shanghai Jiao Tong University,\textsuperscript{\rm2} Shanghai Artificial Intelligence Laboratory\\
  \textsuperscript{\rm3} Shanghai Innovation Institute, \textsuperscript{\rm4} University of Electronic Science and Technology of China\\ 
\texttt{maxwellquadyang@gmail.com}, \texttt{ynylincoln@sjtu.edu.cn}\\
}

\definecolor{top1}{RGB}{212,228,251}
\definecolor{HighLight}{RGB}{212,228,251}
\begin{document}

\maketitle

\begin{abstract}
    Vision-language models enable OOD detection by comparing image alignment with ID labels and negative semantics. Existing negative-label-based methods mainly rely on static negative labels constructed before inference, limiting their ability to cover diverse and evolving OOD concepts. Although test-time expansion provides a natural solution, naively learning negative semantics from potential OOD samples may introduce hard ID contamination. To address this issue, we propose a \textbf{T}est-time \textbf{I}D-prototype-separated \textbf{N}egative \textbf{S}emantics learning method, termed \textbf{TINS}. TINS learns sample-specific negative text embeddings via image-to-text modality inversion and introduces ID-prototype-separated regularization to keep them separated from ID semantics. To further stabilize negative semantics expansion, TINS employs group-wise aggregation scoring and a buffer update strategy. Extensive experiments across Four-OOD, OpenOOD, Temporal-shift, and Various ID settings show consistent improvements over strong baselines. Notably, on the Four-OOD benchmark with ImageNet-1K as ID, TINS reduces the average FPR95 from 14.04\% to 6.72\%. Our code is available at \url{https://github.com/zxk1212/tins}.
    \end{abstract}

\section{Introduction}
Out-of-distribution (OOD) detection \cite{hendrycks17baseline,lee2018simple,liu2020energy} is crucial for deploying vision-language model (VLM)-based classifiers in open-world scenarios. Given only the names of in-distribution (ID) classes, such models are required to classify ID samples while rejecting inputs from unknown OOD categories. This ability is especially important in safety-critical domains, including autonomous driving \cite{wang2025application,asad2025towards} and medical diagnosis \cite{vaya2020bimcv}. 

Recent VLM-based methods tackle this problem by introducing negative labels outside the ID label space \cite{jiang2024negative,chen2024conjugated}, allowing OOD inputs to be detected when they show stronger alignment with negative semantics than with ID semantics. However, these VLM-based methods mainly depend on a static negative-label set constructed before inference. Although effective, these methods rely on static negative semantics that cannot anticipate diverse and evolving OOD concepts during deployment. Since test-time OOD samples naturally reflect the current OOD distribution, a promising direction is to learn negative semantics dynamically at test time.

Nevertheless, recent test-time dynamic methods such as AdaNeg \cite{zhang2024adaneg}, OODD \cite{yang2025oodd}, and InterNeg \cite{xu2026mind} largely overlook an important challenge: 

\textit{
    The unlabeled test stream may include hard ID samples that are easily confused with OOD inputs. 
}

% As a result, these test-time methods may contaminate the negative space with ID-related semantics when they directly expand negative semantics from potential OOD samples, because such samples can contain hard ID instances that are often mistakenly identified as OOD.

As a result, directly expanding negative semantics from potential OOD samples may contaminate the negative space with ID-related semantics.
As shown in Figure~\ref{fig:sun_comparison}, dynamically learned negative features can complement static negative features by aligning more closely with real OOD samples. However, without proper constraints, the resulting ID/OOD score distributions may become confused in the ID-tail region. This suggests that effective test-time dynamic negative learning should be both OOD-adaptive and ID-prototype-separated.

In this paper, we propose a \textbf{T}est-time \textbf{I}D-prototype-separated \textbf{N}egative \textbf{S}emantics learning method, termed \textbf{TINS}. TINS learns sample-specific negative text embeddings from test-time potential OOD samples via image-to-text modality inversion, while imposing ID-prototype-separated regularization to keep the learned embeddings away from ID prototypes. In this way, the negative space can adapt to emerging OOD semantics while mitigating contamination from hard ID samples. We further stabilize the dynamically expanded negative space with group-wise aggregation scoring and a buffer update strategy, enabling robust adaptation to evolving OOD distribution.
Our contributions are summarized as follows:
\begin{itemize}[leftmargin=*,topsep=4pt,itemsep=4pt,parsep=0pt]
    \item We propose ID-prototype-separated regularization to keep learned negative semantics away from ID prototypes, effectively mitigating hard ID contamination.
    
    \item We introduce group-wise aggregation scoring and a buffer update strategy to further stabilize the expansion of negative semantics at test time.
    
    \item Our TINS demonstrates consistent improvements over recent strong baselines across Four-OOD, OpenOOD, Temporal-shift, and Various ID settings. Specifically, on the Four-OOD benchmark, TINS substantially reduces the average FPR95 from 14.04\% to 6.72\% compared with InterNeg.
\end{itemize}

\begin{figure}[t]
    \centering
    \begin{subfigure}[b]{0.32\textwidth}
        \centering
        \includegraphics[width=\textwidth]{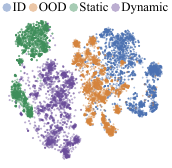}
        \caption{t-SNE visualization}
        \label{fig:tsne_SUN_reg}
    \end{subfigure}
    \hfill
    \begin{subfigure}[b]{0.32\textwidth}
        \centering
        \includegraphics[width=\textwidth]{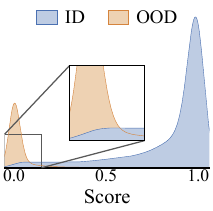}
        \caption{InterNeg}
        \label{fig:id_ood_SUN_wo}
    \end{subfigure}
    \hfill
    \begin{subfigure}[b]{0.32\textwidth}
        \centering
        \includegraphics[width=\textwidth]{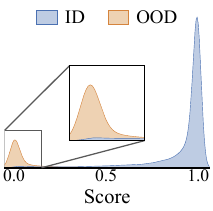}
        \caption{Ours}
        \label{fig:id_ood_SUN_w}
    \end{subfigure}
    \vspace{-1mm}
    \caption{
    Comparison of ID/OOD feature visualizations and distributions using ImageNet-1K as ID and SUN as OOD:
    (a) t-SNE visualization comparing static negative features (\textcolor[HTML]{3E8E5A}{green}) and dynamically learned negative features (\textcolor[HTML]{6B4C9A}{purple}).
    (b) Score distributions of InterNeg~\cite{xu2026mind}, where the lack of proper constraints leads to overlap between ID-tail and OOD samples.
    (c) Score distributions of our proposed TINS, showing clearer ID/OOD separation.
    }
    \label{fig:sun_comparison}
    \vspace{-0.5cm}
\end{figure}

\section{Related Work}
\noindent \textbf{Traditional visual OOD detection.} Traditional visual OOD detection methods mainly rely on supervised ID training and visual feature modeling, and can be broadly grouped into three categories: (1) score-based methods \cite{hendrycks17baseline, lee2018simple, liang2018enhancing, liu2020energy, huang2021mos, wang2022vim}, which design confidence, energy, or logit-space scores to separate ID and OOD samples; (2) distance-based methods \cite{ming2022hyperspherical, zaeemzadeh2021union}, which detect OOD samples by measuring their feature-space discrepancy from ID references, such as class prototypes \cite{lee2018simple, sehwag2021ssd} or nearest ID neighbors \cite{sun2022out, du2022siren, ming2022hyperspherical}; (3) generative-based methods \cite{ryu2018out, kong2021opengan}, which model the ID distribution and regard samples with low likelihood or large reconstruction errors as OOD. 
 % These classical approaches establish the basic detection landscape, but they mostly operate within the visual modality and have limited access to external semantic knowledge.

 \noindent \textbf{OOD detection leveraging pre-trained VLMs.}
With the development of pretrained vision-language models (VLMs), incorporating textual semantics and cross-modal alignment into visual OOD detection has emerged as a new paradigm. Early methods such as LoCoOp \cite{locoop2023} introduce prompt learning to improve few-shot OOD detection by learning ID prompts while exploiting background cues. LSN \cite{nie2024lsn}, ID-Like \cite{bai2024idlike}, and LAPT \cite{zhang2024lapt} further strengthen ID/OOD separability by constructing negative prompts from generated data, auxiliary samples, or local crops. In parallel, training-free VLM-based methods directly use CLIP-like alignment scores for detection: MCM \cite{ming2022delving} evaluates whether an image aligns with the ID label space; EOE~\cite{cao2024eoe} leverages LLMs to generate potential outlier class labels without accessing real OOD samples; NegLabel \cite{jiang2024negative} mines negative labels from external corpora; and CSP~\cite{chen2024conjugated} further enhances negative-label semantics based on NegLabel.

% Earlier method 例如AdaNeg \cite{zhang2024adaneg} and OODD \cite{yang2025oodd} mainly 主要是基于无优化方法的测试时适应方法through 保存测试时潜在的OOD样本的视觉feature。在测试时从词库挖掘潜在的OOD label。 ANTS在测试时从历史样本中动态挖掘疑似 OOD 图像和与其相似的 ID 类别，然后利用外部的 MLLM 的图像理解与推理能力 生成面向 far-OOD和 near-OOD两类negative labels。 InterNeg \cite{xu2026mind}用与 CLIP 训练目标一致的图文跨模态距离来选择和动态生成 negative texts，而不是依赖文本-文本或图像-图像的单模态距离，从而提升 VLM-based OOD detection 的 ID/OOD 区分能力。Different from them, our method xxx
\noindent \textbf{Test-time adaptation methods for OOD detection.}
Recent studies \cite{zhang2024adaneg,yang2025oodd,zhu2025ants,xu2026mind,zhang2026activation} introduce test-time adaptation to OOD detection.
Early methods such as AdaNeg~\cite{zhang2024adaneg} and OODD~\cite{yang2025oodd} mainly follow an optimization-free adaptation paradigm, where potential OOD visual features are cached at test time and update negative proxies.
ANTS~\cite{zhu2025ants} further exploits historical test samples to identify suspicious OOD images and their visually similar ID categories, and leverages the image understanding ability of external MLLMs to generate negative labels for both Far-OOD and Near-OOD.
InterNeg~\cite{xu2026mind} selects and dynamically generates negative labels based on cross-modal image-text distances that are consistent with the CLIP training objective, rather than relying on uni-modal text-text or image-image distances, thereby improving the ID/OOD separability of VLM-based OOD detection.
In contrast, our method learns sample-specific negative text embeddings directly from online test data without external MLLMs, and regularizes them to stay away from ID prototypes.

\section{Methodology}
\subsection{Preliminaries}
\textbf{OOD Detection.} Let $\mathcal{X}$ denote the image space and $\mathcal{Y}^+ = \{ y_1, \dots, y_C \}$ denote the set of ID class labels, where each label is a text token or phrase, e.g., $\mathcal{Y}^+ = \{ \text{cat}, \text{dog}, \dots, \text{bird} \}$. We use $\mathcal{P}_{\mathrm{in}}$ and $\mathcal{P}_{\mathrm{ood}}$ to denote the ID and OOD on $\mathcal{X}$, respectively.
In standard closed-set classification, a test image $x$ is assumed to be drawn from $\mathcal{P}_{\mathrm{in}}$ and assigned a label $y \in \mathcal{Y}^+$. In open-world settings, however, the model may encounter an image $x$ from $\mathcal{P}_{\mathrm{ood}}$, whose labels lie outside $\mathcal{Y}^+$. Such inputs are often forced into one of the known classes, which can lead to unreliable decisions.
OOD detection addresses this issue by retaining the $C$-way ID classifier while introducing an OOD score function $S$ \cite{lee2018simple,liang2018enhancing,liu2020energy} to separate ID and OOD inputs:
\begin{equation}
    G_{\gamma}(x) =
    \begin{cases} 
    \text{ID,} & \text{if } S(x) \geq \gamma; \\
    \text{OOD,} & \text{otherwise,}
    \end{cases}
\end{equation}
where $G_{\gamma}$ is the OOD detector with threshold $\gamma \in \mathbb{R}$, and a larger value of $S$ indicates stronger evidence that $x$ belongs to ID.

\noindent \textbf{CLIP and NegLabel.}
CLIP \cite{radford2021learning} consists of a text encoder $\mathcal{T}(\cdot)$ using the Transformer \cite{vaswani2017attention} architecture and an image encoder $\mathcal{I}(\cdot)$ using the ViT \cite{dosovitskiyimage} or ResNet \cite{he2016deep} architecture. Given a test image $x$, we obtain an image feature $\boldsymbol{v} = \mathcal{I}(x) \in \mathbb{R}^d$ and text features $\boldsymbol{t}_{y_i} = \mathcal{T}(\mathcal{E}(\mathrm{prompt}(y_i))) \in \mathbb{R}^{d}$ for labels $y_i \in \mathcal{Y}^+$, where $\mathrm{prompt}(\cdot)$ represents the prompt template applied to an input label ( e.g., ``a photo of [CLS]''), $\mathcal{E}(\cdot)$ is the word embedding function and $d$ is the feature dimension. Both $\boldsymbol{v}$ and $\boldsymbol{t}_{y_i}$ are $L_2$-normalized. The zero-shot probability of class $y_i$ is then computed as
$
    p_i = \frac{\exp(\boldsymbol{v}^{\top}\boldsymbol{t}_{y_i} / \tau)}{\sum_{j=1}^{C} \exp(\boldsymbol{v}^{\top}\boldsymbol{t}_{y_j} / \tau)},
$
where $\tau > 0$ is the temperature scaling factor.
CLIP has recently been extended to OOD detection \cite{ming2022delving,jiang2024negative,zhang2024lapt}. In particular, Jiang \emph{et al.} \cite{jiang2024negative} construct a set of negative labels $\mathcal{Y}^-$ by selecting labels that are far from the ID data.
Following InterNeg \cite{xu2026mind}, we select $L$ negative labels from the WordNet \cite{miller1995wordnet} corpus $\mathcal{Y}^{\mathrm{cor}}$ based on class-wise image prototypes, forming the negative label set $\mathcal{Y}^- = \{ y^-_{1},y^-_{2}, \dots, y^-_{L}\}$. Specifically, for each ID class $c$, we randomly sample $N$ images $\{x_{cj}\}_{j=1}^{N}$ from the training set (referred to as $N$-shot), where $x_{cj}$ denotes the $j$-th sampled image of class $c$, and compute a class-wise image prototype via the CLIP image encoder:
\begin{equation} \label{equ:class_prototype}
    \boldsymbol{\mu}_c = \frac{1}{N}\sum_{j=1}^{N} \mathcal{I}(x_{cj}), \quad c = 1, 2,\dots, C.
\end{equation}
Then, for each candidate word $y^- \in \mathcal{Y}^{\mathrm{cor}}$, we obtain its text embedding $\boldsymbol{t}_{y^-} = \mathcal{T}(\mathcal{E}(\mathrm{prompt}(y^-)))$ and quantify its semantic distance to the ID visual space by the mean cosine distance to all class-wise image prototypes:
$ \label{equ:neg_distance}
    d(y^-) = \frac{1}{C} \sum_{c=1}^{C} \left(1 - \cos\left(\boldsymbol{t}_{y^-}, \boldsymbol{\mu}_c\right)\right).
$
The negative label set $\mathcal{Y}^- = \{ y^-_{1},y^-_{2}, \dots, y^-_{L}\}$ is subsequently formed by selecting the top-$L$ candidates with the largest distance.
By this construction, each selected negative label is far from the ID visual space in the inter-modal sense. Accordingly, an image input is considered more likely to be ID when it is more similar to ID labels than to these negative labels, yielding the score function:
\begin{equation} \label{equ:neglabel_score}
    S(\boldsymbol{v}) = \frac{\sum_{i=1}^{C} \exp(\boldsymbol{v}^{\top}\boldsymbol{t}_{y_i} / \tau)}{\sum_{j=1}^{C} \exp(\boldsymbol{v}^{\top}\boldsymbol{t}_{y_j} / \tau) + \sum_{j=1}^{L} \exp(\boldsymbol{v}^{\top}\boldsymbol{t}_{y_j^-} / \tau)},
\end{equation}
where $\boldsymbol{t}_{y_j^-} = \mathcal{T}(\mathcal{E}(\mathrm{prompt}(y_j^-))) \in \mathbb{R}^{d}$ is the text feature of the negative label $y_j^-$.

\subsection{Motivation}
\noindent \textbf{Controlled test-time expansion.} Although static negative labels are effective, their semantics are fixed once constructed and cannot adapt to the unknown OOD concepts encountered during deployment. This motivates learning additional negative semantics from potential OOD samples at test time. However, naively expanding the negative bank at test time introduces several challenges: \ding{182} \textbf{Hard ID contamination.} The selection of potential OOD samples is inevitably noisy. Some hard ID samples may receive low score and be selected for expansion. Directly deriving negative semantics from these hard ID samples can contaminate the negative bank with ID-aligned semantics, thereby blurring the boundary between ID and OOD. \ding{183} \textbf{Denominator inflation.} As more test-time learned negatives accumulate in the denominator of Eq. \ref{equ:neglabel_score}, the overall score $S(\boldsymbol{v})$ tends to decrease. In other words, test-time expansion of the negative bank introduces a bank-size effect, which can reduce ID scores and increase false positives.
\ding{184} \textbf{OOD semantic shift.} In streaming data environments, the distribution of OOD data may shift drastically over time, resulting in a mismatch between the negative semantics stored in the bank and those required by the current OOD data.

These challenges motivate a controlled test-time expansion of NegLabel that is both adaptive and stable. Our goal is to exploit potential OOD samples from the current test stream to compensate for the coverage gap of static negative labels, while avoiding hard ID contamination, uncontrolled bank growth, and preserving robustness to OOD semantic drift. To this end, our method is built upon three key components:
\begin{enumerate}[leftmargin=*,topsep=4pt,itemsep=4pt,parsep=0pt,label=\ding{182},ref=\ding{182}]
    \item \textbf{ID-prototype-separated modality inversion} to keep test-time derived negative semantics away from the ID semantic region;
    \item[\ding{183}] \textbf{Group-wise aggregation scoring} to reduce the adverse effect of denominator expansion;
    \item[\ding{184}] \textbf{Buffer update strategy} to maintain diversity and improve robustness to OOD semantic shift.
\end{enumerate}

\subsection{ID-prototype-separated Modality Inversion} \label{sec:reg}

In this section, we introduce a dynamic negative semantics learning framework for potential OOD samples, with an ID-prototype-separated regularization to prevent the learned negative semantics from drifting toward the ID semantic space.

Given a test image $x$, if its OOD score satisfies $S(x)<\beta$, we regard $x$ as a \textbf{potential OOD sample} and learn the negative semantics from it. Since directly learning a discrete textual label is intractable, we instead optimize a negative text embedding via modality inversion. Specifically, we freeze the CLIP model and perform prompt tuning with the template \texttt{`a photo of <CLS>'}, where \texttt{<CLS>} is a pseudo-token with learnable embedding $\boldsymbol{z}$, to map the visual semantics of $x$ into the textual space.
Then, feeding this prompt into the text encoder gives the text feature $\boldsymbol{t}^-$:
\begin{equation}
\boldsymbol{t}^- = \mathcal{T}([\mathcal{E}(\text{``a photo of''}), \boldsymbol{z}]) \in \mathbb{R}^d.
\end{equation}
Previous modality inversion \cite{mistretta2025cross,xu2026mind} optimizes the pseudo-token embedding by simply pulling the learned text feature toward the image feature of the potential OOD sample:
\begin{equation}
\mathcal{L}_{\mathrm{inv}}(\boldsymbol{t}^-, \boldsymbol{v})
=
1-\cos(\boldsymbol{t}^-, \boldsymbol{v}),
\end{equation}
where $\boldsymbol{v}$ denotes the image feature of $x$. However, such a direct inversion objective 
is not sufficiently constrained for dynamic negative semantics expansion. When $x$ is far from the 
ID semantic space, the learned negative feature should naturally align well with its image feature. 
In contrast, when $x$ lies close to the ID semantic space, enforcing strong image-text alignment may 
introduce hard-ID noise into the negative bank.
To address this issue, we optimize $\boldsymbol{z}$ by jointly encouraging OOD alignment and ID-prototype separation:
\begin{equation}
\mathcal{L}_{\mathrm{ours}}(\boldsymbol{t}^-, \boldsymbol{v}) =
1-\cos(\boldsymbol{t}^-, \boldsymbol{v})
+ \lambda \cdot \frac{1}{C}\sum_{c=1}^C \left(1+\cos(\boldsymbol{t}^-, \boldsymbol{\mu}_c)\right),
\label{Eq:our_loss}
\end{equation}
where $\boldsymbol{\mu}_c$ is the prototype of class $c$, and $\lambda$ balances OOD alignment and ID separation. To preserve the prompt prior of CLIP and stabilize optimization, we update only the pseudo-token embedding $\boldsymbol{z}$ while keeping all model parameters and other prompt tokens frozen.

Additionally, we consider two initialization strategies for $\boldsymbol{z}$:

\textbf{(1) \textit{Random initialization.}}
Following \cite{mistretta2025cross,xu2026mind}, we initialize $\boldsymbol{z}$ by sampling from $\mathcal{N}(0, \sigma^2 \mathbf{I})$. This strategy is free from any predefined semantic bias and is therefore more flexible, but our subsequent experiments show that it converges more slowly.

\textbf{(2) \textit{Vocabulary-prior initialization.}}
In this paper, we propose a vocabulary-prior initialization strategy, which initializes $\boldsymbol{z}$ from a corpus word that already provides a favorable trade-off between matching the current sample and staying away from ID prototypes. For each candidate word $y^- $ from the previously constructed negative-label corpus $\mathcal{Y}^-$, we compute the text feature $\boldsymbol{t}_{y^-}$:
\begin{equation}
\boldsymbol{t}_{y^-} = \mathcal{T}(\mathcal{E}(\mathrm{prompt}(y^-))),
\end{equation}
where $\mathrm{prompt}(y^-)$ denotes the prompt template instantiated with $y^-$. We then select
$
y^* = \arg\min_{y^- \in \mathcal{Y}^{-}} \mathcal{L}_{\mathrm{ours}}(\boldsymbol{t}_{y^-}, \boldsymbol{v}),
$
and use the token embedding corresponding to $y^*$ to initialize $\boldsymbol{z}$. This strategy provides a better semantic starting point and empirically leads to faster convergence.
\begin{figure}[t]
    \centering
    \includegraphics[width=\linewidth]{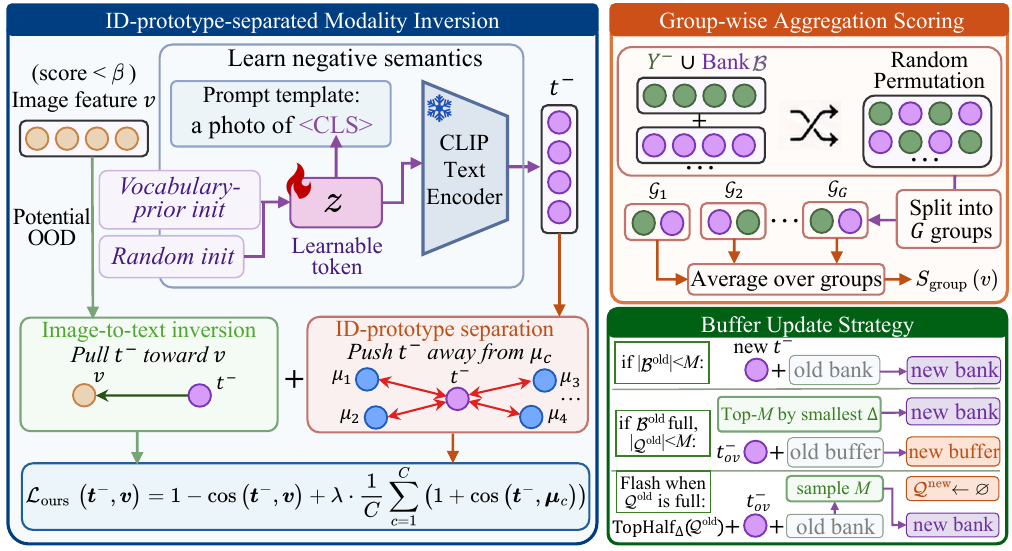}
    \caption{The overall framework of TINS, which learns adaptive negative semantics at test time via ID-prototype-separated modality inversion. The learned negatives are maintained by a dynamic bank and buffer, and combined with static negatives for group-wise aggregation scoring.}
    \label{fig:overview}
\end{figure}

\noindent\textbf{ID-prototype-separated criterion.}
After obtaining the learned negative text feature $\boldsymbol{t}^-$, we filter out ID-like candidates by retaining $\boldsymbol{t}^-$ only if
$
\cos(\boldsymbol{t}^- , \boldsymbol{\mu}_c) < \cos(\boldsymbol{t}_{y_c}, \boldsymbol{\mu}_c),
 \forall c \in \{1,2,\dots,C\},
$
where $\boldsymbol{t}_{y_c}$ denotes the text feature of class $c$. This criterion ensures that the learned negative text feature is less aligned with each ID prototype than its corresponding class text feature.
Furthermore, inspired by OODD \cite{yang2025oodd}, we constrain the capacity of the dynamic negative bank $\mathcal{B}$ to retain only the most ID-prototype-separated learned negative text features. Specifically, we cap the bank size at $M$ and rank each candidate using the same ID-prototype separation term as in our regularization objective,
$
\Delta = \frac{1}{C}\sum_{c=1}^C \left(1+\cos(\boldsymbol{t}^-, \boldsymbol{\mu}_c)\right)
$.
A smaller $\Delta$ indicates stronger separation from ID semantics and thus the candidate is given higher priority to be kept in the dynamic negative bank $\mathcal{B}$.

\subsection{Group-wise Aggregation Scoring}
To mitigate denominator inflation of Eq. \ref{equ:neglabel_score}, we do not place all negative features into a single denominator. Instead, we partition the union of the static negative label set $\mathcal{Y}^-$ and the dynamic negative bank into $G$ disjoint groups of equal size, denoted by $\{\mathcal{G}_g\}_{g=1}^{G}$. For each group, we compute a group-wise OOD score and then average these scores:
\begin{equation}
S_{\mathrm{group}}(\boldsymbol{v})
=
\frac{1}{G}
\sum_{g=1}^{G}
\frac{
\sum_{i=1}^{C}\exp(\boldsymbol{v}^{\top}\boldsymbol{t}_i/\tau)
}{
\sum_{i=1}^{C}\exp(\boldsymbol{v}^{\top}\boldsymbol{t}_i/\tau)
+
C \cdot \frac{1}{|\mathcal{G}_g|}
\sum_{\boldsymbol{t}^- \in \mathcal{G}_g}
\exp(\boldsymbol{v}^{\top}\boldsymbol{t}^-/\tau)
},
\label{Eq:group_score}
\end{equation}
where the factor $\frac{1}{|\mathcal{G}_g|}$ normalizes the contribution of each negative group. Its role is to prevent the negative term from increasing simply due to the growing number of negative features when the dynamic negative bank expands from empty to its maximum capacity $M$. Consequently, the resulting score is less affected by the bank size. The additional scaling factor $C$ further balances the magnitude of the positive and negative terms.

\noindent\textbf{Random permutation.}
Before grouping, we randomly permute the order of whole negative features. This simple strategy reduces ordering bias and encourages semantically diverse negatives to be distributed across groups. The following analysis provides theoretical intuition for the benefit of balanced grouping, and our ablation study confirms that random permutation is empirically effective.

\noindent\textbf{Theoretical insight.} We further provide a theoretical analysis showing that, under a fixed total negative activation, more balanced grouping yields a lower group-wise aggregation score. This makes OOD samples more distinguishable from ID samples in the test stream. Let
$
P = \sum_{i=1}^{C}\exp(\boldsymbol{v}^{\top}\boldsymbol{t}_i/\tau)>0
$
denote the total positive activation, and define the normalized negative activation of group $g$ as
$
A_g =
C \cdot \frac{1}{|\mathcal{G}_g|}
\sum_{\boldsymbol{t}^- \in \mathcal{G}_g}
\exp(\boldsymbol{v}^{\top}\boldsymbol{t}^-/\tau).
$
Then the group-wise aggregation score can be written as
$
S_{\mathrm{group}} =
\frac{1}{G}
\sum_{g=1}^{G}
\frac{P}{P+A_g}.
$
\begin{tcolorbox}[colback=gray!10,%gray background
    colframe=black,% black frame color
    width=\linewidth,% Use 8cm total width,
    arc=1mm, auto outer arc,
    boxrule=0.5pt,
   ]
\begin{theorem}
    Consider two grouping strategies indexed by $k\in\{1,2\}$, each with the same number of groups $G$.
    Let $\mathbf{A}^{(k)}=(A_1^{(k)},A_2^{(k)},\ldots,A_G^{(k)})$ be the group-level negative activations under the grouping strategy $k$,
    and let $A_{[g]}^{(k)}$ denote the $g$-th largest entry of $\mathbf{A}^{(k)}$.
    Under the fixed-total negative activation, the two grouping strategies satisfy
    $
    \sum_{g=1}^{G}A_g^{(1)}=\sum_{g=1}^{G}A_g^{(2)}.
    $
    If $\mathbf{A}^{(2)}$ is more balanced than $\mathbf{A}^{(1)}$, i.e,
    $
    \sum_{g=1}^{m}A_{[g]}^{(2)}
    \le
    \sum_{g=1}^{m}A_{[g]}^{(1)},
    \forall m=1,2,\ldots,G-1,
    $
    then $
    \frac{1}{G}\sum_{g=1}^{G}\frac{P}{P+A_g^{(2)}}
    \le
    \frac{1}{G}\sum_{g=1}^{G}\frac{P}{P+A_g^{(1)}}$.
\label{thm:balanced_grouping}
\end{theorem}
\end{tcolorbox}

The proof of Theorem~\ref{thm:balanced_grouping} is provided in Appendix~\ref{sec:proof}.
Theorem~\ref{thm:balanced_grouping} shows that, for OOD samples, a more balanced grouping of negative features yields a lower group-wise aggregation score under a fixed total negative activation. For ID samples, the proposed ID-prototype-separated regularization helps keep the group-level negative activations $A_g$ relatively small, thereby preserving a high group-wise aggregation score. This contrast further enlarges the separation between ID and OOD samples.

\subsection{Buffer Update Strategy} \label{sec:buffer_update}

Another challenge of the dynamic negative bank is OOD semantic shift. For example, negative labels learned from earlier test samples are effective for the initial OOD semantics, but their relevance may gradually decrease as the distribution of OOD data changes over time. To adapt to newly emerging OOD patterns while retaining useful historical semantics, we introduce a buffer $\mathcal{Q}$ with maximum capacity $M$. The buffer stores candidates displaced from the dynamic bank when it is full, but does not directly participate in the computation of $S_{\mathrm{group}}$.

After a newly learned negative text feature $\boldsymbol{t}^-$ satisfies the ID-prototype-separated criterion, it is used to update the dynamic negative bank $\mathcal{B}$. Let $(\mathcal{B}^{\mathrm{old}}, \mathcal{Q}^{\mathrm{old}})$ denote the bank and buffer before the update, and $(\mathcal{B}^{\mathrm{new}}, \mathcal{Q}^{\mathrm{new}})$ denote those after the update. The update rule is defined as
\begin{equation}
(\mathcal{B}^{\mathrm{new}}, \mathcal{Q}^{\mathrm{new}})
=
\begin{cases}
(\mathcal{B}^{\mathrm{old}} \cup \{\boldsymbol{t}^-\},\; \varnothing),
& \text{if } |\mathcal{B}^{\mathrm{old}}| < M, \\[4pt]
(\mathcal{B}^{\mathrm{top}\text{-}M},\; \mathcal{Q}^{\mathrm{old}} \cup \{\boldsymbol{t}^-_{\mathrm{ov}}\}),
& \text{if } |\mathcal{B}^{\mathrm{old}}| = M \text{ and } |\mathcal{Q}^{\mathrm{old}}| < M, \\[4pt]
\mathrm{Flash}(\mathcal{B}^{\mathrm{old}},\; \mathcal{Q}^{\mathrm{old}} \cup \{\boldsymbol{t}^-_{\mathrm{ov}}\}),
& \text{if } |\mathcal{B}^{\mathrm{old}}| = M \text{ and } |\mathcal{Q}^{\mathrm{old}}| = M,
\end{cases}
\label{Eq:buffer_update}
\end{equation}
where $\mathcal{B}^{\mathrm{top}\text{-}M}$ is obtained by selecting the $M$ candidates with the smallest ID-similarity scores $\Delta$ from $\mathcal{B}^{\mathrm{old}} \cup \{\boldsymbol{t}^-\}$. Since $\mathcal{B}^{\mathrm{old}}$ already contains $M$ candidates, this top-$M$ selection leaves exactly one overflow candidate outside the bank, which we denote as
$
\boldsymbol{t}^{-}_{\mathrm{ov}}
=
\left(\mathcal{B}^{\mathrm{old}} \cup \{\boldsymbol{t}^{-}\}\right)
\setminus
\mathcal{B}^{\mathrm{top}\text{-}M}
$.
When the buffer reaches its maximum capacity $M$, we trigger the $\mathrm{Flash}$ operation:
\begin{equation}
\mathcal{B}^{\mathrm{new}}
\leftarrow
\mathrm{RandSample}
\left(
\mathcal{B}^{\mathrm{old}}
\cup
\mathrm{TopHalf}_{\Delta}(\mathcal{Q}^{\mathrm{old}} \cup \{\boldsymbol{t}^-_{\mathrm{ov}}\}),
M
\right),
\quad
\mathcal{Q}^{\mathrm{new}}
\leftarrow
\varnothing.
\label{Eq:merge_update}
\end{equation}
Specifically, we select the half of the buffered candidates with the smallest $\Delta$ values, and then merge these high-quality buffered candidates with the before-update bank $\mathcal{B}^{\mathrm{old}}$. To maintain the quality of the bank while preventing it from becoming overly concentrated around a narrow subset of negative semantics, we randomly sample $M$ entries from the merged set to form the updated bank instead of always keeping the most ID-prototype-separated candidates. The buffer is then cleared, i.e., $\mathcal{Q}^{\mathrm{new}} \leftarrow \varnothing$.

An overview of our method is shown in Figure \ref{fig:overview} and summarized in Algorithm \ref{alg:tins} of Appendix \ref{sec:pseudo_code}.

\begin{table*}[!b]
    \small
    \centering
    \caption{OOD detection on Four-OOD with ImageNet-1K as ID using CLIP ViT-B/16. $\uparrow$ / $\downarrow$ indicate higher / lower is better. Values are percentages; best values are in \textbf{bold}.}
    \label{tab:ood_results}
    \begin{adjustbox}{width=\textwidth}
    \begin{tabular}{lcccccccccc}
    \toprule
    \multicolumn{1}{l|}{\multirow{3}{*}{Methods}} & \multicolumn{8}{c|}{OOD Datasets}                                                                                                                                          & \multicolumn{2}{c}{\multirow{2}{*}{Average}} \\
    \multicolumn{1}{l|}{}                         & \multicolumn{2}{c}{iNaturalist}     & \multicolumn{2}{c}{SUN}             & \multicolumn{2}{c}{Places}          & \multicolumn{2}{c|}{Textures}                            & \multicolumn{2}{c}{}                         \\ \cmidrule{2-11} 
    \multicolumn{1}{l|}{}                         & AUROC$\uparrow$ & FPR95$\downarrow$ & AUROC$\uparrow$ & FPR95$\downarrow$ & AUROC$\uparrow$ & FPR95$\downarrow$ & AUROC$\uparrow$ & \multicolumn{1}{c|}{FPR95$\downarrow$} & AUROC$\uparrow$      & FPR95$\downarrow$     \\ \midrule
    \multicolumn{11}{c}{\textbf{Visual-based Methods (requiring training on ID or extra data)}}    \\
    \multicolumn{1}{l|}{MSP \cite{hendrycks17baseline}}                      & 87.44           & 58.36             & 79.73           & 73.72             & 79.67           & 74.41             & 79.69           & \multicolumn{1}{c|}{71.93}            & 81.63                & 69.61                 \\
    \multicolumn{1}{l|}{ODIN \cite{liang2018enhancing}}                    & 94.65           & 30.22             & 87.17           & 54.04             & 85.54           & 55.06             & 87.85           & \multicolumn{1}{c|}{51.67}             & 88.80                & 47.75                 \\
    \multicolumn{1}{l|}{Energy \cite{liu2020energy}}                  & 95.33           & 26.12             & 92.66           & 35.97             & 91.41           & 39.87             & 86.76           & \multicolumn{1}{c|}{57.61}             & 91.54                & 39.89                 \\
    \multicolumn{1}{l|}{GradNorm \cite{huang2021importance}}                 & 72.56           & 81.50             & 72.86           & 82.00             & 73.70           & 80.41             & 70.26           & \multicolumn{1}{c|}{79.36}             & 72.35                & 80.82                 \\
    \multicolumn{1}{l|}{ViM \cite{wang2022vim}}                      & 93.16           & 32.19             & 87.19           & 54.01             & 83.75           & 60.67             & 87.18           & \multicolumn{1}{c|}{53.94}             & 87.82                & 50.20                 \\
    \multicolumn{1}{l|}{KNN \cite{sun2022out}}   & 94.52           & 29.17             & 92.67           & 35.62             & 91.02           & 39.61             & 85.67           & \multicolumn{1}{c|}{64.35}             & 90.97                & 42.19                 \\
    \multicolumn{1}{l|}{VOS \cite{du2022vos}}                      & 94.62           & 28.99             & 92.57           & 36.88             & 91.23           & 38.39             & 86.33           & \multicolumn{1}{c|}{61.02}             & 91.19                & 41.32                 \\ \midrule
    \multicolumn{11}{c}{\textbf{VLM-based Methods (requiring training on ID or extra data)}}    \\
    
    \multicolumn{1}{l|}{LoCoOp \cite{locoop2023}}                  & 96.86           & 16.05             & 95.07           & 23.44             & 91.98           & 32.87             & 90.19           & \multicolumn{1}{c|}{42.28}             & 93.52                & 28.66                 \\
    \multicolumn{1}{l|}{LSN \cite{nie2024lsn}}                      & 95.83           & 21.56             & 94.35           & 26.32             & 91.25           & 34.48             & 90.42           & \multicolumn{1}{c|}{38.54}           & 92.96                & 30.22                 \\
    \multicolumn{1}{l|}{ID-Like \cite{bai2024idlike}}             & 98.19           & 8.98             & 91.64           & 42.03             & 90.57           & 44.00            & 94.32           & \multicolumn{1}{c|}{\underline{25.27}}             & 93.68                & 30.07                 \\ 
    \midrule
    \multicolumn{11}{c}{\textbf{VLM-based Methods (no training on ID or extra data)}}    \\
    
    % \multicolumn{1}{l|}{Mahalanobis \cite{lee2018simple}}              & 55.89           & 99.33             & 59.94           & 99.41             & 65.96           & 98.54             & 64.23           & \multicolumn{1}{c|}{98.46}             & 61.50                & 98.94                 \\
    \multicolumn{1}{l|}{Energy \cite{liu2020energy}}                   & 85.09           & 81.08             & 84.24           & 79.02             & 83.38           & 75.08             & 65.56           & \multicolumn{1}{c|}{93.65}             & 79.57                & 82.21                 \\
    \multicolumn{1}{l|}{MCM \cite{ming2022delving}}                      & 94.59           & 32.20             & 92.25           & 38.80             & 90.31           & 46.20             & 86.12           & \multicolumn{1}{c|}{58.50}             & 90.82                & 43.93                 \\
    \multicolumn{1}{l|}{EOE \cite{cao2024eoe}}                     & 97.52           & 12.29             & 95.73           & 20.40             & 92.95           & 30.16             & 85.64           & \multicolumn{1}{c|}{57.53}             & 92.96                & 30.09                 \\
    \multicolumn{1}{l|}{NegLabel \cite{     jiang2024negative}}                 & 99.49           & 1.91              & 95.49           & 20.53             & 91.64           & 35.59             & 90.22           & \multicolumn{1}{c|}{43.56}            & 94.21                & 25.40                 \\
    
    \multicolumn{1}{l|}{CLIPScope \cite{fu2025clipscope}}                     & 99.61           & 1.29             & 96.77           & 15.56             & 93.54           & 28.45             & 91.41           &\multicolumn{1}{c|}{38.37}            & 95.30                & 20.88                 \\
    % \multicolumn{1}{l|}{CoVer \cite{zhang2024cover}}                     & 95.98           & 22.55             & 93.42           & 32.85            & 90.27           & 40.71             & 90.14           & \multicolumn{1}{c|}{43.39}            & 92.45                & 34.88                 \\
    \multicolumn{1}{l|}{CSP \cite{chen2024conjugated}}                     & 99.60           & 1.54             & 96.66           & 13.66            & 92.90           & 29.32             & 93.86           & \multicolumn{1}{c|}{25.52}            & 95.76                & \underline{17.51}                 \\
    \midrule
    \multicolumn{11}{c}{\textbf{Test-time Adaptation Methods}}    \\
    \multicolumn{1}{l|}{AdaNeg \cite{zhang2024adaneg}}                   & \underline{99.71}           & \underline{0.59}              & \underline{97.44}           & \underline{9.50}              & \underline{94.55}           & 34.34             & \underline{94.93}           & \multicolumn{1}{c|}{31.27}             & \underline{96.66}                & 18.92                 \\
    \multicolumn{1}{l|}{OODD \cite{yang2025oodd}} & 99.79           & 0.85              & 97.17           & 12.94             & 92.51          & 30.68          & 94.51           & \multicolumn{1}{c|}{30.67}             & 96.00                & 18.79                 \\
    
    \multicolumn{1}{l|}{InterNeg \cite{xu2026mind}}         & 99.79           & 0.40              & 98.68           & 6.78             & 95.01          & \underline{27.11}          & 96.26           & \multicolumn{1}{c|}{21.85}             & 97.43                & 14.04                 \\
    \multicolumn{1}{l|}{\cellcolor{top1}\textbf{TINS}} & \cellcolor{top1}\textbf{99.93} & \cellcolor{top1}\textbf{0.21} & \cellcolor{top1}\textbf{99.14} & \cellcolor{top1}\textbf{3.84} & \cellcolor{top1}\textbf{97.14} & \cellcolor{top1}\textbf{12.73} & \cellcolor{top1}\textbf{97.83} & \multicolumn{1}{c|}{\cellcolor{top1}\textbf{10.09}} & \cellcolor{top1}\textbf{98.51} & \cellcolor{top1}\textbf{6.72} \\ \bottomrule
    \end{tabular}
    \end{adjustbox}
    \end{table*}

\section{Experiments}
\subsection{Setup}

\textbf{Datasets.} We mainly use ImageNet-1K~\cite{deng2009imagenet} as the ID dataset. Following common practice~\cite{ming2022delving,jiang2024negative,chen2024conjugated,xu2026mind}, we evaluate on Four-OOD benchmark~\cite{van2018inaturalist,xiao2010sun,zhou2017places,cimpoi2014describing} and additionally report results under the OpenOOD benchmark~\cite{zhang2023openood}. We also examine the generality of our method on diverse ID datasets, including Food-101 \cite{bossard2014food}, ImageNet-Sketch \cite{wang2019learning}, ImageNet-R \cite{hendrycks2021many}, and ImageNet-V2 \cite{recht2019imagenet}.

\noindent\textbf{Implementation Details.}
We use CLIP ViT-B/16~\cite{radford2021learning} as the default visual encoder and report alternative backbone results in Appendix \ref{sp:clip}. Following InterNeg~\cite{xu2026mind}, we use $L=2000$ static negative labels, a negative bank capacity of $M=2000$, a batch size of 256, and $N=16$ ID images per class. Based on the hyperparameter analysis on ImageNet-1K with the Four-OOD benchmark, we set $\tau=0.01$, $\beta=0.3$, and $\lambda=0.3$ in Eq.~\ref{Eq:our_loss}. For dynamic negative embedding learning, we adopt the vocabulary-prior initialization by default and optimize the embeddings with AdamW following~\cite{mistretta2025cross,xu2026mind}, using a learning rate of $2\times10^{-2}$, a weight decay of $1\times10^{-2}$, and 30 iterations to balance performance and efficiency. These hyperparameters are fixed unless otherwise specified. We report FPR95, AUROC, following standard evaluation protocols.

\subsection{Main Results}
\noindent \textbf{ImageNet-1K results on the Four-OOD benchmark.}
Table \ref{tab:ood_results} shows that our method clearly outperforms existing methods on Four-OOD with ImageNet-1K as ID. It achieves the best average AUROC of 98.51\% and the lowest average FPR95 of 6.72\%. Compared with InterNeg, the strongest baseline, our method reduces the average FPR95 by 7.32\% and improves the average AUROC by 1.08\%. The improvement is consistent across all OOD datasets, indicating that our dynamic negative-label expansion better captures OOD semantics.

\noindent \textbf{ImageNet-1K results under OpenOOD benchmark.}
As shown in Table \ref{tab:openood_imagenet_adaneg}, our method also performs strongly under the OpenOOD benchmark. Compared with InterNeg, it reduces the FPR95 from 65.43\% to 57.88\% on near-OOD data and from 16.96\% to 12.99\% on far-OOD data. These results confirm that our method remains effective for both challenging near-OOD samples and relatively easier far-OOD samples. We further provide CIFAR-10/100 results in Appendix \ref{sq:openood}, where our advantage remains consistent.

\begin{figure*}[!h]
    \centering

    % ================= Left: Table =================
    \begin{minipage}[t]{0.60\textwidth}
        \centering
        \captionof{table}{OOD detection results under OpenOOD benchmark, where ImageNet-1K is adopted as ID dataset. The detailed results of our method are available in Table \ref{tab:openood_imagenet_full}.}
        \resizebox{\textwidth}{!}{
        \begin{tabular}{l|cc|cc}
        \toprule
        \multirow{2}{*}{Methods} 
        & \multicolumn{2}{|c|}{FPR95 $\downarrow$} 
        & \multicolumn{2}{|c}{$\mathrm{AUROC} \uparrow$} \\
        \cline{2-5}
        & Near-OOD & Far-OOD & Near-OOD & Far-OOD \\
        \midrule
        \multicolumn{5}{c}{\textbf{Methods requiring training on ID or extra data}} \\
        AugMix \cite{hendrycks2019augmix} + ReAct \cite{sun2021react} 
        & -- & -- & 79.94 & 93.70 \\
        SCALE \cite{xu2023scaling} 
        & -- & -- & 81.36 & 96.53 \\
        AugMix \cite{hendrycks2019augmix} + ASH \cite{djurisic2022extremely}   
        & -- & -- & 82.16 & 96.05 \\
        LAPT \cite{zhang2024lapt} 
        & 58.94 & 24.86 & 82.63 & 94.26 \\
        \midrule
        \multicolumn{5}{c}{\textbf{Training-free \& non-adaptive methods}} \\
        MCM \cite{ming2022delving}        
        & 79.02 & 68.54 & 60.11 & 84.77 \\
        NegLabel \cite{jiang2024negative} 
        & 69.45 & 23.73 & 75.18 & 94.85 \\
        \multicolumn{5}{c}{\textbf{Test-time adaptation methods}} \\
        AdaNeg \cite{zhang2024adaneg} 
        & 67.51 & 17.31 & 76.70 & 96.43 \\
        InterNeg \cite{xu2026mind} 
        & 65.43 & 16.96 & 82.20 & 96.71 \\
        \rowcolor{HighLight}  
        \textbf{TINS}  
        & \textbf{57.88} & \textbf{12.99} & \textbf{82.65} & \textbf{96.92} \\
        \bottomrule
        \end{tabular}}
        \label{tab:openood_imagenet_adaneg}
    \end{minipage}
    \hfill
    % ================= Right: Figure =================
    \begin{minipage}[t]{0.39\textwidth}
        \centering
        \vspace{0pt}
        \includegraphics[width=\textwidth]{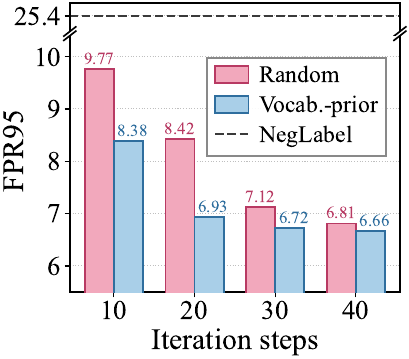}
        \vspace{-0.5cm}
        \captionof{figure}{Effect of two initialization strategies with different iteration steps.}
        \label{fig:interneg_iter_fpr95}
    \end{minipage}

    \vspace{-0.2cm}
\end{figure*}

\noindent\textbf{Various ID datasets.}
To further validate the robustness of our method, we evaluate it on various ID datasets, including Food-101~\cite{bossard2014food}, ImageNet-Sketch~\cite{wang2019learning}, ImageNet-R~\cite{hendrycks2021many}, and ImageNet-V2~\cite{recht2019imagenet}. Following InterNeg~\cite{xu2026mind}, both InterNeg and our method randomly sample four images per class from each ID dataset to construct ID image prototypes, while using the remaining images for evaluation. As shown in Table~\ref{tab:id} (more detailed results can be seen in Appendix \ref{sp:ID}), our method consistently achieves best performance across all ID datasets, demonstrating its robustness to various ID datasets.

\begin{table}[h]
	\small
	\centering
    \vspace{-0.5cm}
	\caption{OOD detection performance comparison on various ID datasets (average over Four-OOD).}
	\label{tab:id}
	\scalebox{0.85}{
		\begin{tabular}{@{}c|cc|cc|cc|cc@{}}
			\toprule
			\multirow{2}{*}{Method} & \multicolumn{2}{c|}{Food-101} & \multicolumn{2}{c|}{ImageNet-Sketch} & \multicolumn{2}{c|}{ImageNet-R} & \multicolumn{2}{c}{ImageNetV2} \\
			& AUROC$\uparrow$ & FPR95$\downarrow$ & AUROC$\uparrow$ & FPR95$\downarrow$ & AUROC$\uparrow$ & FPR95$\downarrow$ & AUROC$\uparrow$ & FPR95$\downarrow$ \\ \midrule
			NegLabel  & 99.89 & 0.41  & 93.59 & 27.42 & 94.55 & 20.63 & 93.08 & 29.77 \\
			CSP       & 99.91 & 0.35  & 95.34 & 19.54 & 97.53 & 8.92  & 95.02 & 20.67 \\
			InterNeg  & 99.98 & 0.13  & 96.59 & 15.66 & 98.20 & 7.38  & 96.71 & 13.55 \\
			TINS     & \textbf{99.99} & \textbf{0.06} & \textbf{99.56} & \textbf{2.03} & \textbf{99.62} & \textbf{1.82} & \textbf{98.20} & \textbf{8.05} \\ \bottomrule
		\end{tabular}}
\end{table}

\subsection{Analyses and Discussions}

\noindent \textbf{Ablation study on key components.}
Table \ref{tab:ablation2} reports the ablation results of group-wise aggregation, random permutation, ID-prototype-separated regularization, and the buffer update strategy on the Four-OOD benchmark with ImageNet-1K as ID.
Removing all components leads to the weakest performance, indicating the limitation of naive dynamic negative-label expansion. Group-wise aggregation significantly improves performance, likely by stabilizing the contribution of negative features. Random permutation further improves performance by reducing order bias in negatives grouping. Regularization also improves performance, suggesting that keeping learned negative embeddings away from ID prototypes helps alleviate hard ID contamination.
The buffer is specifically designed for OOD semantic shift and thus brings only negligible average changes under the default Four-OOD setting, where Case 7 performs almost the same as Case 5. Overall, Case 5 and Case 7 are the two strongest variants, demonstrating that group-wise aggregation, random permutation, and ID-prototype-separated regularization are the main contributors.

    \begin{table*}[!h]
        \centering
        \caption{\textbf{Ablation studies on the key components of our proposed method.} Results are averaged over the Four-OOD benchmark with ImageNet-1K as ID.}
        \label{tab:ablation2}
        \resizebox{0.8\textwidth}{!}{%
        \begin{tabular}{lccccccc}
        \toprule
         & Case 1 & Case 2 & Case 3 & Case 4 & Case 5 & Case 6 & Case 7 \\
        \midrule
        Group-wise Aggregation  & $\times$ & $\checkmark$ & $\checkmark$ & $\checkmark$  & $\checkmark$ & $\checkmark$ & $\checkmark$ \\
        Random Permutation & $\times$ & $\times$ & $\checkmark$ & $\times$ & $\checkmark$ & $\checkmark$ & $\checkmark$ \\
        Regularization & $\times$ & $\times$ & $\times$ & $\checkmark$ & $\checkmark$ & $\times$ & $\checkmark$ \\
        Buffer   & $\times$ & $\times$ & $\times$ & $\times$ & $\times$ & $\checkmark$ & $\checkmark$ \\
        \midrule
        AUROC & 93.67 & 95.46 & 97.81 & 96.41 & 98.50 & 97.77 & 98.51 \\
        FPR95 & 44.32 & 24.00 & 9.87& 18.30 & 6.68 & 9.96 & 6.72 \\
        \bottomrule
        \end{tabular}%
        }
        \end{table*}

\begin{table}[h]
    \centering
    \vspace{-0.3cm}
    \caption{OOD detection results under Temporal-shift, where ImageNet-1K ID dataset and a ViTB/16 CLIP encoder are adopted. Full results are available in Table \ref{tab:temporal_shifts_vitb16_dup}.}
    \label{tab:temporal_shifts_ov}
    \resizebox{\textwidth}{!}{
    \begin{tabular}{lcccccccccc}
        \toprule
        \multirow{2}{*}{Methods}
        & \multicolumn{2}{c}{\makebox[0.3em][c]{\textit{i}}~$\rightarrow$~\makebox[0.3em][c]{\textit{S}}~$\rightarrow$~\makebox[0.3em][c]{\textit{P}}~$\rightarrow$~\makebox[0.3em][c]{\textit{T}}}
        & \multicolumn{2}{c}{\makebox[0.3em][c]{\textit{S}}~$\rightarrow$~\makebox[0.3em][c]{\textit{P}}~$\rightarrow$~\makebox[0.3em][c]{\textit{T}}~$\rightarrow$~\makebox[0.3em][c]{\textit{i}}}
        & \multicolumn{2}{c}{\makebox[0.3em][c]{\textit{P}}~$\rightarrow$~\makebox[0.3em][c]{\textit{T}}~$\rightarrow$~\makebox[0.3em][c]{\textit{i}}~$\rightarrow$~\makebox[0.3em][c]{\textit{S}}}
        & \multicolumn{2}{c}{\makebox[0.3em][c]{\textit{T}}~$\rightarrow$~\makebox[0.3em][c]{\textit{i}}~$\rightarrow$~\makebox[0.3em][c]{\textit{S}}~$\rightarrow$~\makebox[0.3em][c]{\textit{P}}}
        & \multicolumn{2}{c}{Overall Average} \\
        \cmidrule(lr){2-3} \cmidrule(lr){4-5} \cmidrule(lr){6-7} \cmidrule(lr){8-9} \cmidrule(lr){10-11}
        & AUROC$\uparrow$ & FPR95$\downarrow$
        & AUROC$\uparrow$ & FPR95$\downarrow$
        & AUROC$\uparrow$ & FPR95$\downarrow$
        & AUROC$\uparrow$ & FPR95$\downarrow$
        & AUROC$\uparrow$ & FPR95$\downarrow$ \\
        \midrule
        w/o buffer & 96.97 & 14.35 & 96.74 & 14.92 & 96.90 & 13.95 & 97.18 & 13.23 & 96.95 & 14.11 \\
       w/ buffer & \textbf{97.83} & \textbf{10.17} & \textbf{97.72} & \textbf{10.28} & \textbf{97.72} & \textbf{10.40} & \textbf{98.28} & \textbf{7.54}  & \textbf{97.89} & \textbf{9.60}  \\
   
        \bottomrule
    \end{tabular}
    }
    \vspace{-0.3cm}
\end{table}

\noindent \textbf{Robustness under Temporal-shift.}
To evaluate robustness under sudden semantic shift in the evolving distribution of OOD data, 
we construct the Temporal-shift setting where Four-OOD datasets (iNaturalist, SUN, Places, Textures, denoted as \textit{i}, \textit{S}, \textit{P}, and \textit{T},  respectively) arrive sequentially in different orderings. 
Crucially, the learned negative embeddings accumulated during inference are not reset as initial state. 
% \begin{wrapfigure}{r}{0.4\textwidth}
%     \vspace{-0.4cm}
%     \centering
%     \includegraphics[width=0.36\textwidth]{fig/interneg_iter_fpr95.pdf}
%     \vspace{-0.2cm}
%     \caption{Effect of two initialization strategies with different iteration steps.}
%     \label{fig:interneg_iter_fpr95}
%     \vspace{-0.8cm}
% \end{wrapfigure}
As shown in Table~\ref{tab:temporal_shifts_ov}, using the buffer consistently improves all temporal orders. This demonstrates that the buffer improves the stability of dynamic learning of negative-label adaptation under OOD semantic shift. In addition, we analyze the impact of ID/OOD data ordering in Appendix \ref{sec:data_order}.

\noindent \textbf{Effect of Vocabulary-prior Initialization.}
Figure \ref{fig:interneg_iter_fpr95} compares random initialization with the proposed vocabulary-prior initialization under different optimization steps. The vocabulary-prior strategy consistently achieves lower FPR95 across all iteration steps, demonstrating that it provides a more favorable starting point for modality inversion. Compared with random initialization, it reaches strong performance with fewer iteration steps, indicating faster convergence and better optimization efficiency.
\begin{figure*}[h]
    \centering
    \begin{minipage}[t]{0.33\textwidth}
        \centering
        \includegraphics[width=\linewidth]{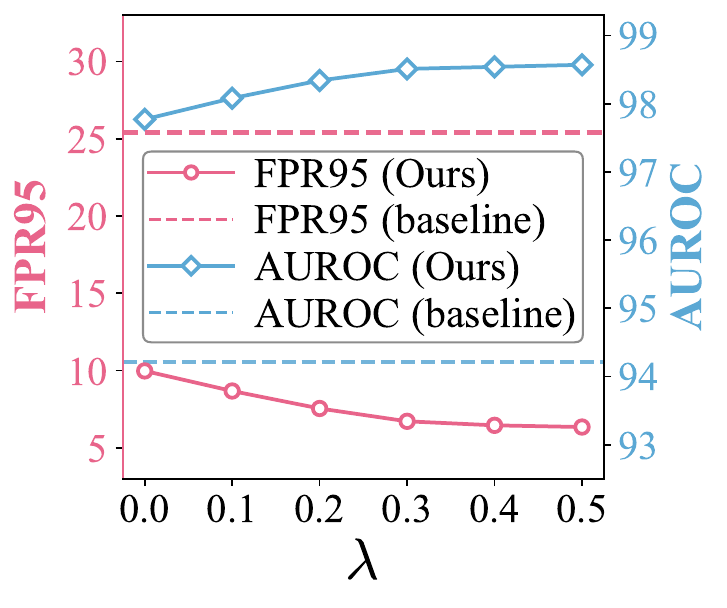}
        \vspace{-6mm}
        \caption*{(a) Effect of the coefficient $\lambda$}
    \end{minipage}\hspace{-0.4mm}
    \begin{minipage}[t]{0.33\textwidth}
        \centering
        \includegraphics[width=\linewidth]{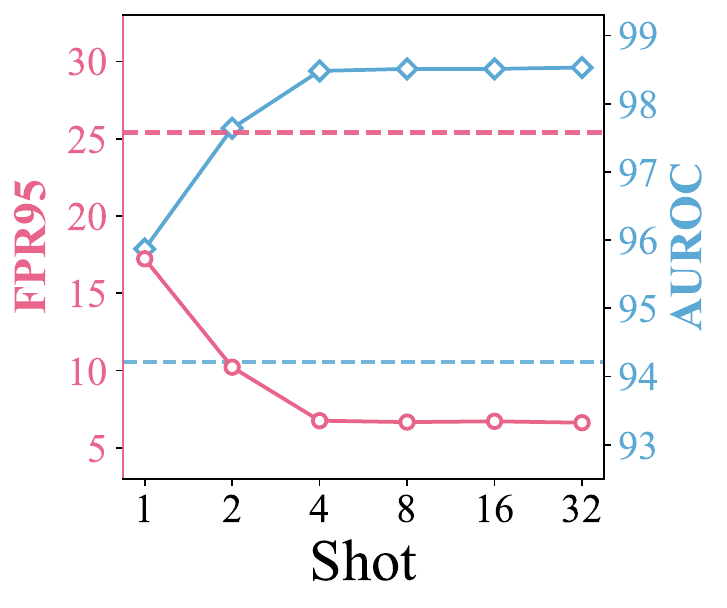}
        \vspace{-6mm}
        \caption*{(b) Effect of ID-shot number}
    \end{minipage}\hspace{-0.4mm}
    \begin{minipage}[t]{0.33\textwidth}
        \centering
        \includegraphics[width=\linewidth]{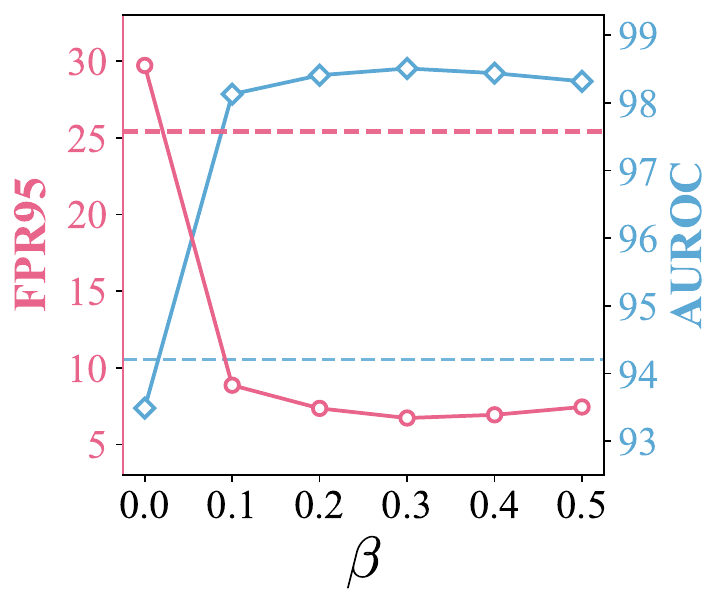}
        \vspace{-6mm}
        \caption*{(c) Effect of threshold $\beta$}
    \end{minipage}
    \vspace{-3mm}
    \caption{Hyperparameter analyses: (a) Regularization coefficient $\lambda$; (b) Number of ID shots for prototype estimation; (c) Threshold $\beta$ on detection performance. Results are averaged over the Four-OOD benchmark with ImageNet-1K as ID.}
    \label{fig:ablation}
\end{figure*}

\noindent \textbf{Hyperparameter Analyses.}
As shown in Figure \ref{fig:ablation}, we observe that our method is insensitive to key hyperparameters. The performance improves with a moderate regularization coefficient $\lambda$ and becomes relatively stable around $\lambda=0.3$, confirming the effect of ID-prototype separation. Increasing the number of ID shots strengthens prototype estimation, while the gain becomes marginal after 4 shots. For $\beta$, dynamic expansion is ineffective when $\beta=0$, but remains stable across a wide range of positive thresholds. We therefore use $\lambda=0.3$, 16 shots, and $\beta=0.3$ as default settings. 

Additional analyses of other hyperparameters are provided in Appendix~\ref{sup:hyper}, where our method also shows stable performance. We further evaluate our method under different ID:OOD ratios in Appendix~\ref{sup:ratios} and study its sensitivity to corpus choice in Appendix~\ref{sup:corpus}.

\section{Conclusion}

We propose a \textbf{T}est-time \textbf{I}D-prototype-separated \textbf{N}egative \textbf{S}emantics learning method, termed \textbf{TINS}, for OOD detection. 
TINS dynamically learns adaptive negative embeddings from online test samples, enabling the OOD detector to better capture evolving OOD semantics. 
To stabilize this process, we introduce ID-prototype-separated regularization to suppress hard-ID contamination, group-wise aggregation scoring to alleviate denominator inflation, and a buffer update strategy to improve robustness under OOD semantic shift, achieving state-of-the-art results on various benchmarks.

\bibliographystyle{plainnat}
\bibliography{main}

@String(CVPR= {IEEE Conf. Comput. Vis. Pattern Recog.})

@String(ICLR = {Int. Conf. Learn. Represent.})

@String(CVPR  = {CVPR})

@String(ICLR  = {ICLR})

@article{wang2025application,
  title={Application of Uncertainty to Out-of-Distribution Detection for Autonomous Driving Perception Safety},
  author={Wang, Ke and Ma, Qi and Shen, Chongqiang and Lu, Jianbo},
  journal={IEEE Transactions on Intelligent Transportation Systems},
  year={2025},
  publisher={IEEE}
}

@article{asad2025towards,
  title={Towards Robust Autonomous Driving: Out-of-Distribution Object Detection in Bird's Eye View Space},
  author={Asad, Muhammad and Ullah, Ihsan and Sistu, Ganesh and Madden, Michael G},
  journal={IEEE Open Journal of Vehicular Technology},
  year={2025},
  publisher={IEEE}
}

@article{vaswani2017attention,
  title={Attention is all you need},
  author={Vaswani, Ashish and Shazeer, Noam and Parmar, Niki and Uszkoreit, Jakob and Jones, Llion and Gomez, Aidan N and Kaiser, {\L}ukasz and Polosukhin, Illia},
  journal={Advances in neural information processing systems},
  volume={30},
  year={2017}
}

@article{hendrycks17baseline,
  author    = {Dan Hendrycks and Kevin Gimpel},
  title     = {A Baseline for Detecting Misclassified and Out-of-Distribution Examples in Neural Networks},
  journal = {Proceedings of International Conference on Learning Representations},
  year = {2017},
}

@article{huang2021importance,
  title={On the importance of gradients for detecting distributional shifts in the wild},
  author={Huang, Rui and Geng, Andrew and Li, Yixuan},
  journal={Advances in Neural Information Processing Systems},
  volume={34},
  pages={677--689},
  year={2021}
}

@article{sun2021react,
  title={React: Out-of-distribution detection with rectified activations},
  author={Sun, Yiyou and Guo, Chuan and Li, Yixuan},
  journal={Advances in Neural Information Processing Systems},
  volume={34},
  pages={144--157},
  year={2021}
}

@inproceedings{xu2023scaling,
  title={Scaling for Training Time and Post-hoc Out-of-distribution Detection Enhancement},
  author={Xu, Kai and Chen, Rongyu and Franchi, Gianni and Yao, Angela},
  booktitle={The Twelfth International Conference on Learning Representations}
}

@article{du2022vos,
      title={VOS: Learning What You Don’t Know by Virtual Outlier Synthesis}, 
      author={Du, Xuefeng and Wang, Zhaoning and Cai, Mu and Li, Yixuan},
      journal={Proceedings of the International Conference on Learning Representations},
      year={2022}
}

@inproceedings{kong2021opengan,
  title={Opengan: Open-set recognition via open data generation},
  author={Kong, Shu and Ramanan, Deva},
  booktitle={Proceedings of the IEEE/CVF International Conference on Computer Vision},
  pages={813--822},
  year={2021}
}

@article{hendrycks2019augmix,
  title={Augmix: A simple data processing method to improve robustness and uncertainty},
  author={Hendrycks, Dan and Mu, Norman and Cubuk, Ekin D and Zoph, Barret and Gilmer, Justin and Lakshminarayanan, Balaji},
  journal={arXiv preprint arXiv:1912.02781},
  year={2019}
}

@InProceedings{Hendrycks_2022_CVPR,
    author    = {Hendrycks, Dan and Zou, Andy and Mazeika, Mantas and Tang, Leonard and Li, Bo and Song, Dawn and Steinhardt, Jacob},
    title     = {PixMix: Dreamlike Pictures Comprehensively Improve Safety Measures},
    booktitle = {Proceedings of the IEEE/CVF Conference on Computer Vision and Pattern Recognition (CVPR)},
    month     = {June},
    year      = {2022},
    pages     = {16783-16792}
}

@article{hendrycks2019using,
  title={Using self-supervised learning can improve model robustness and uncertainty},
  author={Hendrycks, Dan and Mazeika, Mantas and Kadavath, Saurav and Song, Dawn},
  journal={Advances in neural information processing systems},
  volume={32},
  year={2019}
}

@article{hendrycks2018deep,
  title={Deep Anomaly Detection with Outlier Exposure},
  author={Hendrycks, Dan and Mazeika, Mantas and Dietterich, Thomas},
  journal={Proceedings of the International Conference on Learning Representations},
  year={2019}
}

@article{liu2020energy,
  title={Energy-based out-of-distribution detection},
  author={Liu, Weitang and Wang, Xiaoyun and Owens, John and Li, Yixuan},
  journal={Advances in neural information processing systems},
  volume={33},
  pages={21464--21475},
  year={2020}
}

@article{krizhevsky2009learning,
  title={Learning multiple layers of features from tiny images},
  author={Krizhevsky, Alex and Hinton, Geoffrey and others},
  year={2009},
  publisher={Toronto, ON, Canada}
}

@article{deng2012mnist,
  title={The mnist database of handwritten digit images for machine learning research [best of the web]},
  author={Deng, Li},
  journal={IEEE signal processing magazine},
  volume={29},
  number={6},
  pages={141--142},
  year={2012},
  publisher={IEEE}
}

@article{netzer2011reading,
  title={Reading digits in natural images with unsupervised feature learning},
  author={Netzer, Yuval and Wang, Tao and Coates, Adam and Bissacco, Alessandro and Wu, Bo and Ng, Andrew Y},
  year={2011}
}

@inproceedings{cimpoi2014describing,
  title={Describing textures in the wild},
  author={Cimpoi, Mircea and Maji, Subhransu and Kokkinos, Iasonas and Mohamed, Sammy and Vedaldi, Andrea},
  booktitle={Proceedings of the IEEE conference on computer vision and pattern recognition},
  pages={3606--3613},
  year={2014}
}

@inproceedings{xiao2010sun,
  title={Sun database: Large-scale scene recognition from abbey to zoo},
  author={Xiao, Jianxiong and Hays, James and Ehinger, Krista A and Oliva, Aude and Torralba, Antonio},
  booktitle={2010 IEEE computer society conference on computer vision and pattern recognition},
  pages={3485--3492},
  year={2010},
  organization={IEEE}
}

@article{zhou2017places,
  title={Places: A 10 million image database for scene recognition},
  author={Zhou, Bolei and Lapedriza, Agata and Khosla, Aditya and Oliva, Aude and Torralba, Antonio},
  journal={IEEE transactions on pattern analysis and machine intelligence},
  volume={40},
  number={6},
  pages={1452--1464},
  year={2017},
  publisher={IEEE}
}

@article{le2015tiny,
  title={Tiny imagenet visual recognition challenge},
  author={Le, Ya and Yang, Xuan},
  journal={CS 231N},
  volume={7},
  number={7},
  pages={3},
  year={2015}
}

@inproceedings{deng2009imagenet,
  title={Imagenet: A large-scale hierarchical image database},
  author={Deng, Jia and Dong, Wei and Socher, Richard and Li, Li-Jia and Li, Kai and Fei-Fei, Li},
  booktitle={2009 IEEE conference on computer vision and pattern recognition},
  pages={248--255},
  year={2009},
  organization={Ieee}
}

@inproceedings{bossard2014food,
  title={Food-101--mining discriminative components with random forests},
  author={Bossard, Lukas and Guillaumin, Matthieu and Van Gool, Luc},
  booktitle={European conference on computer vision},
  pages={446--461},
  year={2014},
  organization={Springer}
}

@article{wang2019learning,
  title={Learning robust global representations by penalizing local predictive power},
  author={Wang, Haohan and Ge, Songwei and Lipton, Zachary and Xing, Eric P},
  journal={Advances in neural information processing systems},
  volume={32},
  year={2019}
}

@inproceedings{recht2019imagenet,
  title={Do imagenet classifiers generalize to imagenet?},
  author={Recht, Benjamin and Roelofs, Rebecca and Schmidt, Ludwig and Shankar, Vaishaal},
  booktitle={International conference on machine learning},
  pages={5389--5400},
  year={2019},
  organization={PMLR}
}

@inproceedings{hendrycks2021many,
  title={The many faces of robustness: A critical analysis of out-of-distribution generalization},
  author={Hendrycks, Dan and Basart, Steven and Mu, Norman and Kadavath, Saurav and Wang, Frank and Dorundo, Evan and Desai, Rahul and Zhu, Tyler and Parajuli, Samyak and Guo, Mike and others},
  booktitle={Proceedings of the IEEE/CVF international conference on computer vision},
  pages={8340--8349},
  year={2021}
}

@article{lee2018simple,
  title={A simple unified framework for detecting out-of-distribution samples and adversarial attacks},
  author={Lee, Kimin and Lee, Kibok and Lee, Honglak and Shin, Jinwoo},
  journal={Advances in neural information processing systems},
  volume={31},
  year={2018}
}

@inproceedings{sun2022out,
  title={Out-of-distribution detection with deep nearest neighbors},
  author={Sun, Yiyou and Ming, Yifei and Zhu, Xiaojin and Li, Yixuan},
  booktitle={International Conference on Machine Learning},
  pages={20827--20840},
  year={2022},
  organization={PMLR}
}

@inproceedings{djurisic2022extremely,
  title={Extremely Simple Activation Shaping for Out-of-Distribution Detection},
  author={Djurisic, Andrija and Bozanic, Nebojsa and Ashok, Arjun and Liu, Rosanne},
  booktitle={The Eleventh International Conference on Learning Representations}
}

@inproceedings{zhang2024lapt,
  title={Lapt: Label-driven automated prompt tuning for ood detection with vision-language models},
  author={Zhang, Yabin and Zhu, Wenjie and He, Chenhang and Zhang, Lei},
  booktitle={European conference on computer vision},
  pages={271--288},
  year={2024},
  organization={Springer}
}

@inproceedings{liang2018enhancing,
  title={Enhancing The Reliability of Out-of-distribution Image Detection in Neural Networks},
  author={Liang, Shiyu and Li, Yixuan and Srikant, R},
  booktitle={International Conference on Learning Representations},
  year={2018}
}

@inproceedings{zhang2023openood,
  title={OpenOOD v1. 5: Enhanced Benchmark for Out-of-Distribution Detection},
  author={Zhang, Jingyang and Yang, Jingkang and Wang, Pengyun and Wang, Haoqi and Lin, Yueqian and Zhang, Haoran and Sun, Yiyou and Du, Xuefeng and Li, Yixuan and Liu, Ziwei and others},
  booktitle={NeurIPS 2023 Workshop on Distribution Shifts: New Frontiers with Foundation Models}
}

@inproceedings{huang2021mos,
  title={Mos: Towards scaling out-of-distribution detection for large semantic space},
  author={Huang, Rui and Li, Yixuan},
  booktitle={Proceedings of the IEEE/CVF Conference on Computer Vision and Pattern Recognition},
  pages={8710--8719},
  year={2021}
}

@inproceedings{vaze2021open,
  title={Open-set recognition: A good closed-set classifier is all you need},
  author={Vaze, Sagar and Han, Kai and Vedaldi, Andrea and Zisserman, Andrew},
  booktitle={International conference on learning representations},
  year={2021}
}

@inproceedings{bitterwolf2023or,
  title={In or out? fixing imagenet out-of-distribution detection evaluation},
  author={Bitterwolf, Julian and M{\"u}ller, Maximilian and Hein, Matthias},
  booktitle={Proceedings of the 40th International Conference on Machine Learning},
  pages={2471--2506},
  year={2023}
}

@inproceedings{yang2025oodd,
  title={OODD: Test-time Out-of-Distribution Detection with Dynamic Dictionary},
  author={Yang, Yifeng and Zhu, Lin and Sun, Zewen and Liu, Hengyu and Gu, Qinying and Ye, Nanyang},
  booktitle={Proceedings of the Computer Vision and Pattern Recognition Conference},
  pages={30630--30639},
  year={2025}
}

@inproceedings{wang2022vim,
  title={Vim: Out-of-distribution with virtual-logit matching},
  author={Wang, Haoqi and Li, Zhizhong and Feng, Litong and Zhang, Wayne},
  booktitle={Proceedings of the IEEE/CVF conference on computer vision and pattern recognition},
  pages={4921--4930},
  year={2022}
}

@inproceedings{van2018inaturalist,
  title={The inaturalist species classification and detection dataset},
  author={Van Horn, Grant and Mac Aodha, Oisin and Song, Yang and Cui, Yin and Sun, Chen and Shepard, Alex and Adam, Hartwig and Perona, Pietro and Belongie, Serge},
  booktitle={Proceedings of the IEEE conference on computer vision and pattern recognition},
  pages={8769--8778},
  year={2018}
}

@inproceedings{he2016deep,
  title={Deep residual learning for image recognition},
  author={He, Kaiming and Zhang, Xiangyu and Ren, Shaoqing and Sun, Jian},
  booktitle={Proceedings of the IEEE conference on computer vision and pattern recognition},
  pages={770--778},
  year={2016}
}

@article{ming2022delving,
  title={Delving into out-of-distribution detection with vision-language representations},
  author={Ming, Yifei and Cai, Ziyang and Gu, Jiuxiang and Sun, Yiyou and Li, Wei and Li, Yixuan},
  journal={Advances in neural information processing systems},
  volume={35},
  pages={35087--35102},
  year={2022}
}

@inproceedings{jiang2024negative,
  title={Negative Label Guided OOD Detection with Pretrained Vision-Language Models},
  author={Jiang, Xue and Liu, Feng and Fang, Zhen and Chen, Hong and Liu, Tongliang and Zheng, Feng and Han, Bo},
  booktitle={The Twelfth International Conference on Learning Representations},
  year={2024}
}

@inproceedings{radford2021learning,
  title={Learning transferable visual models from natural language supervision},
  author={Radford, Alec and Kim, Jong Wook and Hallacy, Chris and Ramesh, Aditya and Goh, Gabriel and Agarwal, Sandhini and Sastry, Girish and Askell, Amanda and Mishkin, Pamela and Clark, Jack and others},
  booktitle={International conference on machine learning},
  pages={8748--8763},
  year={2021},
  organization={PmLR}
}

@article{zhang2024adaneg,
  title={Adaneg: Adaptive negative proxy guided ood detection with vision-language models},
  author={Zhang, Yabin and Zhang, Lei},
  journal={Advances in Neural Information Processing Systems},
  volume={37},
  pages={38744--38768},
  year={2024}
}

@inproceedings{fu2025clipscope,
  title={Clipscope: Enhancing zero-shot ood detection with bayesian scoring},
  author={Fu, Hao and Patel, Naman and Krishnamurthy, Prashanth and others},
  booktitle={Proceedings of the Winter Conference on Applications of Computer Vision},
  pages={5346--5355},
  year={2025}
}

@inproceedings{dosovitskiyimage,
  title={An Image is Worth 16x16 Words: Transformers for Image Recognition at Scale},
  author={Dosovitskiy, Alexey and Beyer, Lucas and Kolesnikov, Alexander and Weissenborn, Dirk and Zhai, Xiaohua and Unterthiner, Thomas and Dehghani, Mostafa and Minderer, Matthias and Heigold, Georg and Gelly, Sylvain and others},
  booktitle={International Conference on Learning Representations}
}

@inproceedings{locoop2023,
  author       = {Atsuyuki Miyai and
                  Qing Yu and
                  Go Irie and
                  Kiyoharu Aizawa},
  title        = {LoCoOp: Few-Shot Out-of-Distribution Detection via Prompt Learning},
  booktitle    = {Annual Conference
                  on Neural Information Processing Systems},
  pages = {76298--76310},
  year         = {2023},
}

@inproceedings{bai2024idlike,
  author       = {Yichen Bai and
                  Zongbo Han and
                  Bing Cao and
                  Xiaoheng Jiang and
                  Qinghua Hu and
                  Changqing Zhang},
  title        = {ID-like Prompt Learning for Few-Shot Out-of-Distribution Detection},
  booktitle    = {Conference on Computer Vision and Pattern Recognition},
  pages        = {17480--17489},
  year         = {2024},
}

@inproceedings{nie2024lsn,
  author       = {Jun Nie and
                  Yonggang Zhang and
                  Zhen Fang and
                  Tongliang Liu and
                  Bo Han and
                  Xinmei Tian},
  title        = {Out-of-Distribution Detection with Negative Prompts},
  booktitle    = {International Conference on Learning Representations},
  pages        = {1--20},
  year         = {2024},
}

@inproceedings{cao2024eoe,
  author       = {Chentao Cao and
                  Zhun Zhong and
                  Zhanke Zhou and
                  Yang Liu and
                  Tongliang Liu and
                  Bo Han},
  title        = {Envisioning Outlier Exposure by Large Language Models for Out-of-Distribution
                  Detection},
  booktitle    = {International Conference on Machine Learning},
  year         = {2024},
}

@inproceedings{chen2024conjugated,
  author       = {Mengyuan Chen and
                  Junyu Gao and
                  Changsheng Xu},
  title        = {Conjugated Semantic Pool Improves {OOD} Detection with Pre-trained Vision-Language Models},
  booktitle    = {Annual Conference on Neural Information Processing Systems},
  pages = {82560--82593},
  year         = {2024},
}

@article{xu2026mind,
  title={Mind the Way You Select Negative Texts: Pursuing the Distance Consistency in OOD Detection with VLMs},
  author={Xu, Zhikang and Xu, Qianqian and Wang, Zitai and Hua, Cong and Li, Sicong and Yang, Zhiyong and Huang, Qingming},
  journal={arXiv preprint arXiv:2603.02618},
  year={2026}
}

@article{han2024aucseg,
  title={Aucseg: Auc-oriented pixel-level long-tail semantic segmentation},
  author={Han, Boyu and Xu, Qianqian and Yang, Zhiyong and Bao, Shilong and Wen, Peisong and Jiang, Yangbangyan and Huang, Qingming},
  journal={Advances in Neural Information Processing Systems},
  volume={37},
  pages={126863--126907},
  year={2024}
}

@article{ming2022hyperspherical,
  title={How to Exploit Hyperspherical Embeddings for Out-of-Distribution Detection?},
  author={Ming, Yifei and Sun, Yiyou and Dia, Ousmane and Li, Yixuan},
  journal={arXiv preprint arXiv:2203.04450},
  year={2022}
}

@inproceedings{zaeemzadeh2021union,
  title={Out-of-Distribution Detection Using Union of 1-Dimensional Subspaces},
  author={Zaeemzadeh, Alireza and Bisagno, Niccolo and Sambugaro, Zeno and Conci, Nicola and Rahnavard, Nazanin and Shah, Mubarak},
  booktitle={Proceedings of the IEEE/CVF Conference on Computer Vision and Pattern Recognition},
  pages={9452--9461},
  year={2021}
}

@inproceedings{ryu2018out,
  title={Out-of-Domain Detection Based on Generative Adversarial Network},
  author={Ryu, Seonghan and Koo, Sangjun and Yu, Hwanjo and Lee, Gary Geunbae},
  booktitle={Proceedings of the 2018 Conference on Empirical Methods in Natural Language Processing},
  pages={714--718},
  year={2018}
}

@article{sehwag2021ssd,
  title={SSD: A Unified Framework for Self-Supervised Outlier Detection},
  author={Sehwag, Vikash and Chiang, Mung and Mittal, Prateek},
  journal={arXiv preprint arXiv:2103.12051},
  year={2021}
}

@inproceedings{du2022siren,
  author       = {Xuefeng Du and
                  Gabriel Gozum and
                  Yifei Ming and
                  Yixuan Li},
  title        = {{SIREN:} Shaping Representations for Detecting Out-of-Distribution
                  Objects},
  booktitle    = {Annual Conference
                  on Neural Information Processing Systems},
  pages = {20434--20449},
  year         = {2022},
}

@inproceedings{mistretta2025cross,
  title={CROSS THE GAP: EXPOSING THE INTRA-MODAL MISALIGNMENT IN CLIP VIA MODALITY INVERSION},
  author={Mistretta, M and Baldrati, A and Agnolucci, L and Bertini, M and Bagdanov, AD and others},
  booktitle={13th International Conference on Learning Representations, ICLR 2025},
  pages={90437--90458},
  year={2025},
  organization={International Conference on Learning Representations, ICLR}
}

@inproceedings{liu2023gen,
  title={Gen: Pushing the limits of softmax-based out-of-distribution detection},
  author={Liu, Xixi and Lochman, Yaroslava and Zach, Christopher},
  booktitle={Proceedings of the IEEE/CVF conference on computer vision and pattern recognition},
  pages={23946--23955},
  year={2023}
}

@article{zhu2025ants,
  title={ANTS: Adaptive Negative Textual Space Shaping for OOD Detection via Test-Time MLLM Understanding and Reasoning},
  author={Zhu, Wenjie and Zhang, Yabin and Jin, Xin and Zeng, Wenjun and Zhang, Lei},
  journal={arXiv preprint arXiv:2509.03951},
  year={2025}
}

@article{vaya2020bimcv,
    title={BIMCV COVID-19+: a large annotated dataset of RX and CT images from COVID-19 patients},
    author={Vay{\'a}, Maria De La Iglesia and Saborit, Jose Manuel and Montell, Joaquim Angel and Pertusa, Antonio and Bustos, Aurelia and Cazorla, Miguel and Galant, Joaquin and Barber, Xavier and Orozco-Beltr{\'a}n, Domingo and Garc{\'\i}a-Garc{\'\i}a, Francisco and others},
    journal={arXiv preprint arXiv:2006.01174},
    year={2020}
  }

@article{zhang2026activation,
  title={Activation Matters: Test-time Activated Negative Labels for OOD Detection with Vision-Language Models},
  author={Zhang, Yabin and Varma, Maya and Gao, Yunhe and Delbrouck, Jean-Benoit and Liu, Jiaming and Wang, Chong and Langlotz, Curtis},
  journal={arXiv preprint arXiv:2603.25250},
  year={2026}
}

@article{miller1995wordnet,
  title={WordNet: a lexical database for English},
  author={Miller, George A},
  journal={Communications of the ACM},
  volume={38},
  number={11},
  pages={39--41},
  year={1995},
  publisher={ACM New York, NY, USA}
}

%%%%%%%%%%%%%%%%%%%%%%%%%%%%%%%%%%%%%%%%%%%%%%%%%%%%%%%%%%%%
\setcounter{theorem}{0}  % 重置定理计数器，使附录中的定理从1开始编号

\appendix
\newpage
\section*{Appendices}
\addcontentsline{toc}{section}{Appendix}
\startcontents[appendices]
\printcontents[appendices]{}{1}{\setcounter{tocdepth}{2}}
\newpage
\section{Limitations}
\begin{wrapfigure}[15]{r}{0.42\textwidth}
    \vspace{-0.6cm}
    \centering
    \includegraphics[width=0.95\linewidth]{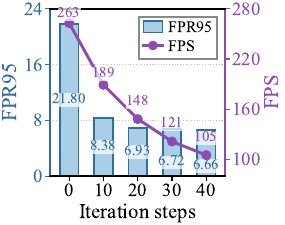}
    \vspace{-0.25cm}
    \caption{Trade-off between OOD detection performance and inference speed under different optimization iteration steps.}
    \label{fig:fps}
    \vspace{-0.4cm}
\end{wrapfigure}
One minor limitation of our method is the additional cost introduced by test-time optimization.
As shown in Figure~\ref{fig:fps}, on a single NVIDIA RTX 4090 GPU, increasing the number of optimization iteration steps from 0 to 40 reduces FPR95 from 21.80\% to 6.66\%, while decreasing the inference speed from 263 FPS to 105 FPS.
In addition, optimizing prompt embeddings incurs extra memory overhead, with a peak additional memory usage of 2076 MB in our experiments.
This performance--efficiency trade-off may limit the applicability of our method in latency- or memory-sensitive deployment scenarios.
This overhead may be alleviated by hardware-aware acceleration, e.g., mixed-precision execution on Tensor Cores for the repeated text-encoder forward/backward passes during modality inversion.

\section{Broader Impacts}
This work aims to improve the reliability of vision-language models in open-world deployment by detecting inputs outside the in-distribution label space. Improved OOD detection can benefit safety-critical applications such as autonomous driving, medical image analysis, and content moderation, where overconfident predictions on unknown inputs may cause harmful decisions. By adapting negative semantics at test time, our method may help deployed systems handle evolving data distributions during test time. However, OOD detection is not a complete safeguard. False negatives may accept OOD samples as ID, while false positives may reject valid ID inputs. In high-stakes domains, our method should therefore be used as part of a broader safety pipeline with human oversight, calibration, monitoring, and domain-specific validation. Since it relies on pretrained VLMs and corpus-derived negative semantics, it may also inherit biases from the underlying model and vocabulary source.

\section{Pseudo Code} \label{sec:pseudo_code}
\IncMargin{1em}
\begin{algorithm}[h]
\caption{The proposed TINS method for OOD detection}
\label{alg:tins}
\LinesNumbered
\KwIn{ID label space $\mathcal{Y}^{+}$ and test set $\mathcal{X}$;}
$\mathcal{Y}^{-} \leftarrow$ static negative labels mined from corpus\;
Construct ID prototypes $\{\boldsymbol{\mu}_c\}_{c=1}^{C}$\;
Initialize an empty dynamic bank $\mathcal{B}$ and buffer $\mathcal{Q}$\;
\For{$\boldsymbol{x} \in \mathcal{X}$}{
    Compute $S_{\mathrm{group}}(\boldsymbol{x})$ with static negatives and the current dynamic bank\;
    Determine whether $\boldsymbol{x}$ is potential OOD by thresholding current $S_{\mathrm{group}}(\boldsymbol{x})$\;
    \If{$\boldsymbol{x}$ is potential OOD}{
        Initialize the pseudo-token embedding $\boldsymbol{z}$ with random or vocabulary-prior initialization\;
        Learn negative text feature $\boldsymbol{t}^{-}$ via image-to-text modality inversion with $\mathcal{L}_{\mathrm{ours}}$\; 
        Filter $\boldsymbol{t}^{-}$ with ID-prototype-separated criterion\;
        Update dynamic bank $\mathcal{B}$ and buffer $\mathcal{Q}$ using Eq.~\ref{Eq:buffer_update}\;
    }
    Recompute $S_{\mathrm{group}}(\boldsymbol{x})$ with static negatives and the updated dynamic bank\;
}
\Return{Collected recomputed scores $\{S_{\mathrm{group}}\}$}\;
\end{algorithm}
\DecMargin{1em}
\section{Proof of Theorem~\ref{thm:balanced_grouping}} \label{sec:proof}

\begin{theorem}
    Consider two grouping strategies indexed by $k\in\{1,2\}$, each with the same number of groups $G$.
    Let $\mathbf{A}^{(k)}=(A_1^{(k)},\ldots,A_G^{(k)})$ be the group-level negative activations under the grouping strategy $k$,
    and let $A_{[g]}^{(k)}$ denote the $g$-th largest entry of $\mathbf{A}^{(k)}$.
    Under the fixed-total negative activation, the two grouping strategies satisfy
    $
    \sum_{g=1}^{G}A_g^{(1)}=\sum_{g=1}^{G}A_g^{(2)}.
    $
    If $\mathbf{A}^{(2)}$ is more balanced than $\mathbf{A}^{(1)}$, i.e,
    $
    \sum_{g=1}^{m}A_{[g]}^{(2)}
    \le
    \sum_{g=1}^{m}A_{[g]}^{(1)},
    \forall m=1,2,\ldots,G-1,
    $
    then $
    \frac{1}{G}\sum_{g=1}^{G}\frac{P}{P+A_g^{(2)}}
    \le
    \frac{1}{G}\sum_{g=1}^{G}\frac{P}{P+A_g^{(1)}}$.
    \end{theorem}
\begin{proof}
    Define
\begin{equation}
    \phi(a)=\frac{P}{P+a},\qquad a\ge 0.
    \end{equation}
    
    Since  $P>0$ ,
    \begin{equation}
    \phi'(a)=-\frac{P}{(P+a)^2}<0, \qquad \phi''(a)=\frac{2P}{(P+a)^3}>0,
    \end{equation}
    so  $\phi$  is decreasing and strictly convex on  $[0,\infty)$ .
    
    We first establish the following elementary fact: for any $x \geq y \geq 0$, if the total mass $x+y$ is fixed, then transferring a small positive $\delta$ from the larger entry $x$ to the smaller entry $y$ cannot increase $\phi(x)+\phi(y)$. Let  $x\ge y\ge 0$ , and for any  $0\le \delta\le (x-y)/2$ , define
    \begin{equation}
    x'=x-\delta,\qquad y'=y+\delta.
    \end{equation}
    Consider
    \begin{equation}
    h(t)=\phi(x-t)+\phi(y+t), \qquad 0\le t\le (x-y)/2.
    \end{equation}
    Then
    \begin{equation}
    h'(t)=\phi'(y+t)-\phi'(x-t)\le 0,
    \end{equation}
    because  $\phi'$  is increasing and  $y+t\le x-t$ . Hence
    \begin{equation}
    \phi(x-\delta)+\phi(y+\delta)\le \phi(x)+\phi(y).
    \end{equation}
    Thus, transferring mass from a larger component to a smaller one never increases  $\sum_g \phi(A_g)$ .
    
    Now, since  $\mathbf A^{(2)}$  is more balanced than  $\mathbf A^{(1)}$  and both vectors have the same total sum,  $\mathbf A^{(2)}$  can be obtained from  $\mathbf A^{(1)}$  through a finite sequence of such pairwise balancing operations: each step moves some amount from a larger entry to a smaller one while preserving the total sum. Applying the two-point inequality at every step yields
    \begin{equation}
    \sum_{g=1}^G \phi(A_g^{(2)})\le \sum_{g=1}^G \phi(A_g^{(1)}).
    \end{equation}
    Therefore,
    \begin{equation}
    \frac{1}{G}\sum_{g=1}^{G}\frac{P}{P+A^{(2)}_g} \le \frac{1}{G}\sum_{g=1}^{G}\frac{P}{P+A^{(1)}_g},
    \end{equation}
    This completes the proof.  
\end{proof}

\section{More Detailed Results}
%更多的超参数实验结果
\subsection{Statistical Significance}
We control the test-time data ordering with different random seeds and report OOD detection results on the Four-OOD benchmark, using ImageNet-1K as the ID dataset. Across three random seeds, our method achieves an AUROC of 98.51 ± 0.04 and an FPR95 of 6.80 ± 0.09, demonstrating its robustness to variations in test-time data ordering.

\subsection{More Hyperparameter Experiments} \label{sup:hyper}

\begin{figure*}[h]
    \centering
    \begin{minipage}[t]{0.33\textwidth}
        \centering
        \includegraphics[width=\linewidth]{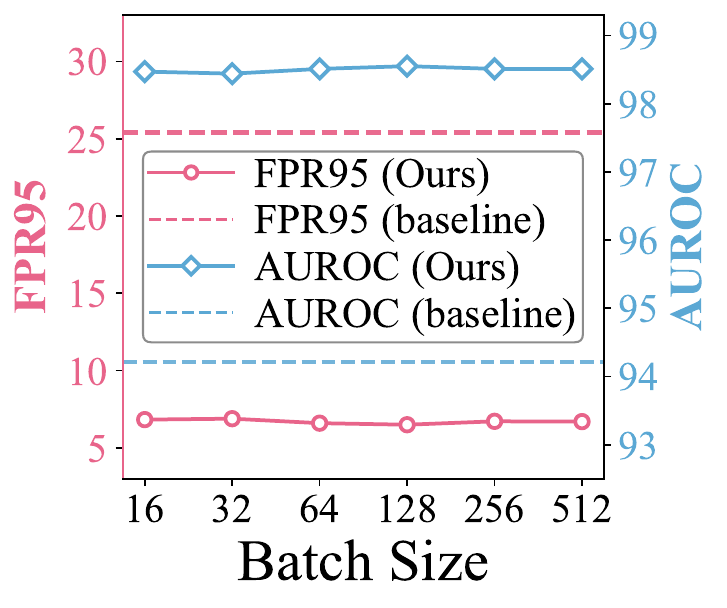}
        \vspace{-5mm}
        \caption*{(a) Effect of batch size}
    \end{minipage}\hspace{-2mm}
    \begin{minipage}[t]{0.33\textwidth}
        \centering
        \includegraphics[width=\linewidth]{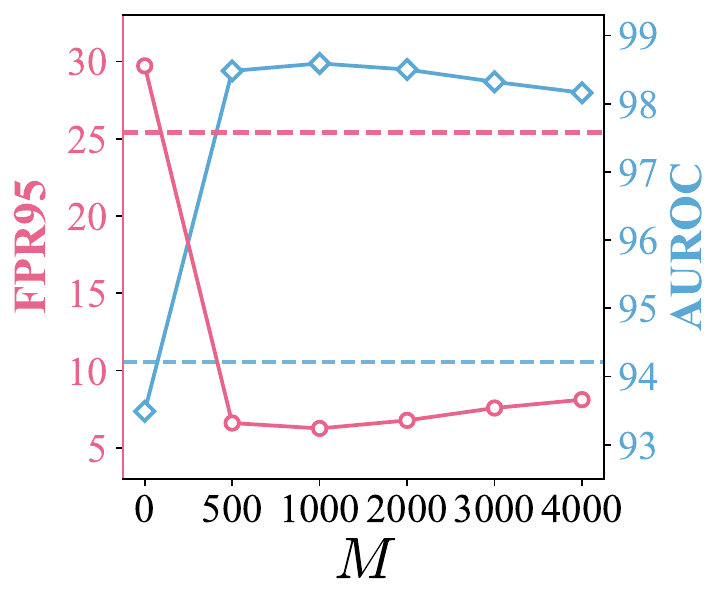}
        \vspace{-5mm}
        \caption*{(b) Effect of bank capacity}
    \end{minipage}\hspace{-2mm}
    \begin{minipage}[t]{0.33\textwidth}
        \centering
        \includegraphics[width=\linewidth]{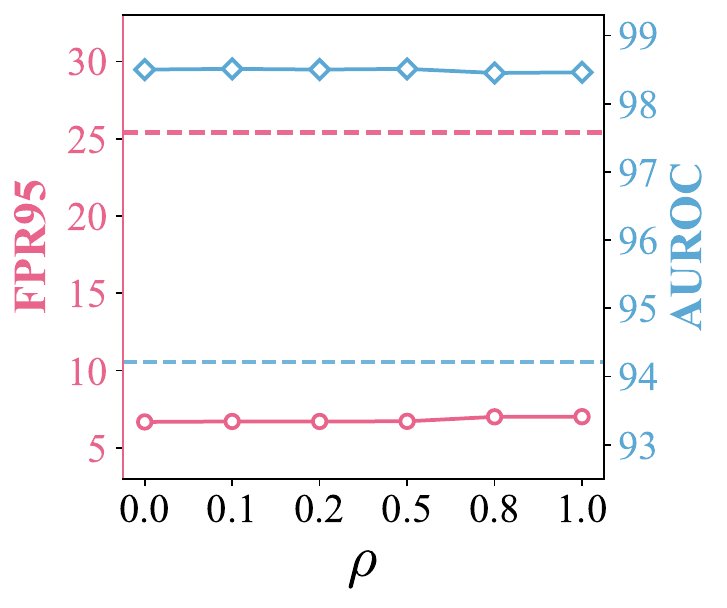}
        \vspace{-5mm}
        \caption*{(c) Effect of buffer selection ratio}
    \end{minipage}
    \vspace{-2mm}
    \caption{Ablation studies on (a) inference batch size, (b) bank capacity $M$, and (c) buffer selection ratio $\rho$. Results are averaged over the Four-OOD benchmark with ImageNet-1K as ID.}
    \label{fig:ablation_three}
    \vspace{-0.5cm}
\end{figure*}

In Section \ref{sec:buffer_update}, we select the half of the buffer with the smallest ID-similarity scores $\Delta$ during the Flash operation. More generally, this design can be formulated with a buffer selection ratio $\rho \in [0,1]$, where $\mathrm{Top}_{\rho M,\Delta}(\mathcal{Q})$ denotes the $\rho M$ buffered candidates with the smallest $\Delta$ values. The Flash operation then becomes
\begin{equation}
\mathcal{B}^{\mathrm{new}}
\leftarrow
\mathrm{RandSample}
\left(
\mathcal{B}^{\mathrm{old}}
\cup
\mathrm{Top}_{\rho M,\Delta}(\mathcal{Q}^{\mathrm{old}} \cup \{\boldsymbol{t}^-_{\mathrm{ov}}\}),
M
\right),
\quad
\mathcal{Q}^{\mathrm{new}}
\leftarrow
\varnothing.
\end{equation}
Figure~\ref{fig:ablation_three} further shows the effects of batch size, bank capacity $M$, and buffer selection ratio $\rho$. We observe that: (1) our method remains stable across different batch sizes, indicating that prototype estimation is insensitive to the mini-batch configuration; (2) increasing $M$ improves performance over the static baseline, while the gains quickly saturate, suggesting that a moderate bank capacity is sufficient to capture useful OOD semantics; and (3) varying $\rho$ has little effect on performance, showing that the Flash operation is robust to the buffer selection ratio. We therefore set $\rho=0.5$ by default, which balances the reuse of high-quality embeddings in the bank with adaptation to newly emerging OOD semantics.

We also analyze the sensitivity to the group number \(G\) in Table \ref{tab:group_factor}. 
The performance improves substantially when moving from \(G=1\) to multiple groups and remains stable for \(G \in \{2,5,8,10\}\). 
This confirms that group-wise aggregation effectively alleviates denominator inflation and is not sensitive to the exact choice of \(G\).
\begin{table}[h]
    \centering
    \vspace{-0.5cm}
    \caption{Sensitivity analysis of the group number $G$.}
    \begin{tabular}{l|ccccc}
        \toprule
        $G$ & 1 & 2 & 5 & 8 & 10 \\
        \midrule
        AUROC  & 93.67 & 98.26 & 98.51 & 98.41 & 98.29 \\
        \bottomrule
    \end{tabular}
    \label{tab:group_factor}
    \vspace{-0.5cm}
\end{table}

\subsection{Detailed results on various ID datasets} \label{sp:ID}
\begin{table}[h]
	\small
	\centering
	\caption{OOD detection performance comparison on various ID datasets.}
	\label{tab:backbone_d}
	\scalebox{0.7}{
		\begin{tabular}{@{}c|c|cccccccc|cc@{}}
			\toprule
			\multirow{3}{*}{ID datasets} & \multirow{3}{*}{Methods} & \multicolumn{8}{c|}{OOD Datasets} & \multicolumn{2}{c}{\multirow{2}{*}{Average}} \\
			&  & \multicolumn{2}{c}{iNaturalist} & \multicolumn{2}{c}{SUN} & \multicolumn{2}{c}{Places} & \multicolumn{2}{c|}{Textures} & \multicolumn{2}{c}{} \\ \cmidrule(l){3-12} 
			&  & AUROC$\uparrow$ & FPR95$\downarrow$ & AUROC$\uparrow$ & FPR95$\downarrow$ & AUROC$\uparrow$ & FPR95$\downarrow$ & AUROC$\uparrow$ & FPR95$\downarrow$ & AUROC$\uparrow$ & FPR95$\downarrow$ \\ \midrule
			\multirow{4}{*}{Food-101} 
				& NegLabel & 99.99 & 0.01 & 99.99 & 0.01 & 99.99 & 0.01 & 99.60 & 1.61 & 99.89 & 0.41 \\
				& CSP & \textbf{100.00} & \textbf{0.00} & \textbf{100.00} & \textbf{0.00} & 99.99 & 0.01 & 99.63 & 1.40 & 99.91 & 0.35 \\
				& InterNeg & \textbf{100.00} & \textbf{0.00} & \textbf{100.00} & \textbf{0.00} & \textbf{100.00} & \textbf{0.00} & 99.91 & 0.51 & 99.98 & 0.13 \\
				& Ours & \textbf{100.00} & \textbf{0.00} & \textbf{100.00} & \textbf{0.00} & \textbf{100.00} & \textbf{0.00} & \textbf{99.94} & \textbf{0.25} & \textbf{99.99} & \textbf{0.06} \\ \midrule
			\multirow{4}{*}{ImageNet-Sketch} 
				& NegLabel & 99.34 & 2.24 & 94.93 & 22.73 & 90.78 & 38.62 & 89.29 & 46.10 & 93.59 & 27.42 \\
				& CSP & 99.49 & 1.60 & 96.41 & 15.30 & 92.51 & 31.41 & 92.95 & 29.86 & 95.34 & 19.54 \\
				& InterNeg & 99.36 & 2.61 & 99.33 & 2.45 & 97.76 & 8.73 & 89.93 & 48.85 & 96.59 & 15.66 \\
				& Ours & \textbf{99.97} & \textbf{0.13} & \textbf{99.97} & \textbf{0.12} & \textbf{99.81} & \textbf{0.75} & \textbf{98.50} & \textbf{7.11} & \textbf{99.56} & \textbf{2.03} \\ \midrule
			\multirow{4}{*}{ImageNet-R} 
				& NegLabel & 99.58 & 1.60 & 96.03 & 15.77 & 91.97 & 29.48 & 90.60 & 35.67 & 94.55 & 20.63 \\
				& CSP & 99.79 & 0.89 & 98.49 & 6.16 & 95.41 & 18.46 & 96.44 & 10.16 & 97.53 & 8.92 \\
				& InterNeg & 99.90 & 0.38 & 99.50 & 2.09 & 96.73 & 14.00 & 96.65 & 13.07 & 98.20 & 7.38 \\
				& Ours & \textbf{99.99} & \textbf{0.02} & \textbf{99.95} & \textbf{0.13} & \textbf{99.58} & \textbf{1.66} & \textbf{98.95} & \textbf{5.46} & \textbf{99.62} & \textbf{1.82} \\ \midrule
			\multirow{4}{*}{ImageNetV2} 
				& NegLabel & 99.40 & 2.47 & 94.46 & 25.69 & 90.00 & 42.03 & 88.46 & 48.90 & 93.08 & 29.77 \\
				& CSP & 99.54 & 1.76 & 96.10 & 17.16 & 91.66 & 34.12 & 92.76 & 29.65 & 95.02 & 20.67 \\
				& InterNeg & 99.82 & 0.60 & 98.32 & 8.47 & 94.34 & 23.52 & 94.37 & 21.61 & 96.71 & 13.55 \\
				& Ours & \textbf{99.90} & \textbf{0.36} & \textbf{99.16} & \textbf{4.32} & \textbf{96.52} & \textbf{16.21} & \textbf{97.21} & \textbf{11.29} & \textbf{98.20} & \textbf{8.05} \\ \bottomrule
	\end{tabular}}
\end{table}

\subsection{Detailed OOD detection results on the OpenOOD benchmark} \label{sq:openood}

Table \ref{tab:openood_imagenet_full} provides the detailed breakdown of our ImageNet-1K results under the OpenOOD benchmark. 
Our method achieves strong performance on both near-OOD and far-OOD splits, with particularly low FPR95 on far-OOD datasets such as iNaturalist and Textures. 

For CIFAR-10/100, AdaNeg~\cite{zhang2024adaneg} observes that using a larger number of static negative labels $L$ yields better performance. 
Following this setting, we set $L$ to 70,000 while keeping all other hyperparameters unchanged, and further evaluate our method on CIFAR-100 and CIFAR-10 as ID datasets in Tables~\ref{tab:ood_results_openood_cifar100}--\ref{tab:full_cifar10}. 
On CIFAR-100, our method substantially improves over InterNeg, reducing near-OOD FPR95 from 62.54\% to 33.44\% and far-OOD FPR95 from 20.02\% to 2.90\%. 
On CIFAR-10, our method also achieves the best performance. 
These results demonstrate that our method perform well beyond ImageNet-1K and remains effective across different ID label spaces.

\begin{table}[!h]
  \centering
\caption{Detailed OOD detection results on the OpenOOD benchmark, where ImageNet-1K is adopted as ID dataset.}\label{tab:openood_imagenet_full}
\begin{tabular}{ll|c|c}
\toprule
Settings & OOD Datasets & FPR95 $\downarrow$ & AUROC $\uparrow$\\
\midrule
\multirow{3}{*}{Near-OOD} & SSB-hard \cite{vaze2021open} & 61.68 & 81.38 \\
                        & NINCO \cite{bitterwolf2023or}   & 54.09 & 83.91 \\
                        & \textbf{Mean}  & \textbf{57.88} & \textbf{82.65} \\
\midrule     
\multirow{4}{*}{Far-OOD} & iNaturalist \cite{van2018inaturalist} & 0.21 & 99.93 \\
                        & Textures   \cite{cimpoi2014describing} &10.09  & 97.83 \\
                        & OpenImage-O  \cite{wang2022vim} & 28.68 & 92.99 \\
                        & \textbf{Mean} & \textbf{12.99} & \textbf{96.92} \\
\bottomrule
\end{tabular}
\vspace{-0.2cm}
\end{table}

% As illustrated in Tab. \ref{tab:openood_cifar_adaneg}, the advantage of our AdaNeg also holds on the CIFAR10/100 dataset. Notably, our method achieves new state-of-the-art results in the far-OOD setting under a zero-shot training-free manner, even outperforming its competitors training on the full labeled training set.

\begin{table}[H]
    \small
    \centering
    \caption{OOD detection results with ID dataset of CIFAR-100 on the OpenOOD benchmark using CLIP ViT-B/16 architecture. Full results are available in Table \ref{tab:full_cifar100}.}
    \label{tab:ood_results_openood_cifar100}
    \begin{tabular}{lcccc}
    \toprule
    \multicolumn{1}{l|}{\multirow{2}{*}{Methods}} & \multicolumn{2}{c|}{FPR95 $\downarrow$} & \multicolumn{2}{c}{AUROC $\uparrow$}  \\ \cmidrule{2-5}  
    \multicolumn{1}{l|}{} & Near-OOD  & \multicolumn{1}{c|}{Far-OOD}  & Near-OOD & Far-OOD  \\
    \midrule
    \multicolumn{5}{c}{\textcolor{gray}{\textbf{Methods requiring training on ID or extra data}}} \\
    
    \multicolumn{1}{l|}{\textcolor{gray}{GEN \cite{liu2023gen}}} & \textcolor{gray}{–} & \multicolumn{1}{c|}{\textcolor{gray}{–}} & \textcolor{gray}{81.31} & \textcolor{gray}{79.68}  \\
    \multicolumn{1}{l|}{\textcolor{gray}{VOS \cite{du2022vos} + EBO \cite{liu2020energy}}} & \textcolor{gray}{–} & \multicolumn{1}{c|}{\textcolor{gray}{–}} & \textcolor{gray}{80.93} & \textcolor{gray}{81.32}  \\
    \multicolumn{1}{l|}{\textcolor{gray}{SCALE \cite{xu2023scaling}}}  & \textcolor{gray}{–} & \multicolumn{1}{c|}{\textcolor{gray}{–}} & \textcolor{gray}{80.99} & \textcolor{gray}{81.42} \\
    \multicolumn{1}{l|}{\textcolor{gray}{OE \cite{hendrycks2018deep} + MSP \cite{hendrycks17baseline}}}  & \textcolor{gray}{–} & \multicolumn{1}{c|}{\textcolor{gray}{–}} & \textcolor{gray}{88.30} & \textcolor{gray}{81.41} \\
    \midrule
    \multicolumn{5}{c}{\textbf{Zero-shot methods (no training on ID or extra data)}} \\
    
    \multicolumn{1}{l|}{MCM \cite{ming2022delving}} & 75.20 & \multicolumn{1}{c|}{59.32} & 71.00 & 76.00\\
    \multicolumn{1}{l|}{NegLabel \cite{jiang2024negative}} & 71.44 & \multicolumn{1}{c|}{40.92} & 70.58 & 89.68  \\
    \multicolumn{1}{l|}{AdaNeg \cite{zhang2024adaneg}}  & 59.07 & \multicolumn{1}{c|}{29.35} & 84.60 & 95.25 \\
    \multicolumn{1}{l|}{InterNeg\cite{xu2026mind}}  & 62.54 & \multicolumn{1}{c|}{20.02} & 85.45 & 96.39 \\
    \multicolumn{1}{l|}{\cellcolor{top1}\textbf{Ours}}  & \cellcolor{top1}\textbf{33.44} & \multicolumn{1}{c|}{\cellcolor{top1}\textbf{2.90}} & \cellcolor{top1}\textbf{88.36} & \cellcolor{top1}\textbf{99.05} \\
    \bottomrule
    \end{tabular}
    \end{table}

    \begin{table}[h!]
    \small
    \centering
    \caption{Full results of our method with ID dataset of CIFAR-100 on the OpenOOD benchmark.}
    \label{tab:full_cifar100}
    \begin{tabular}{l|l|c|c}
    \toprule
    Near / Far OOD & Datasets & FPR95 $\downarrow$ & AUROC $\uparrow$ \\
    \midrule
    \multirow{3}{*}{Near-OOD} & CIFAR-10 \cite{krizhevsky2009learning} & 51.32 & 82.25 \\
     & TIN \cite{le2015tiny} & 15.57 & 94.47 \\
     & \textbf{Mean} & \textbf{33.44} & \textbf{88.36} \\
    \midrule
    \multirow{5}{*}{Far-OOD} & MNIST \cite{deng2012mnist} & 0.00 & 100.00 \\
     & SVHN \cite{netzer2011reading} & 0.00 & 100.00 \\
     & Texture / DTD \cite{cimpoi2014describing} & 3.14 & 98.90 \\
     & Places365 \cite{zhou2017places} & 8.47 & 97.30 \\
     & \textbf{Mean} & \textbf{2.90} & \textbf{99.05} \\
    \bottomrule
    \end{tabular}
    \end{table}
    
    \begin{table}[!h]
    \small
    \centering
    \caption{OOD detection results with ID dataset of CIFAR-10 on the OpenOOD benchmark using CLIP ViT-B/16 architecture. Full results are available in Table \ref{tab:full_cifar10}.}
    \label{tab:ood_results_openood_cifar10}
    \begin{tabular}{lcccc}
    \toprule
    \multicolumn{1}{l|}{\multirow{2}{*}{Methods}} & \multicolumn{2}{c|}{FPR95 $\downarrow$} & \multicolumn{2}{c}{AUROC $\uparrow$}  \\ \cmidrule{2-5}  
    \multicolumn{1}{l|}{} & Near-OOD  & \multicolumn{1}{c|}{Far-OOD}  & Near-OOD & Far-OOD  \\
    \midrule
    \multicolumn{5}{c}{\textcolor{gray}{\textbf{Methods requiring training on ID or extra data}}} \\
    
    \multicolumn{1}{l|}{\textcolor{gray}{PixMix \cite{Hendrycks_2022_CVPR} + KNN \cite{sun2022out}}} & \textcolor{gray}{–} & \multicolumn{1}{c|}{\textcolor{gray}{–}} & \textcolor{gray}{93.10} & \textcolor{gray}{95.94}  \\
    \multicolumn{1}{l|}{\textcolor{gray}{OE \cite{hendrycks2018deep} + MSP \cite{hendrycks17baseline}}} & \textcolor{gray}{–} & \multicolumn{1}{c|}{\textcolor{gray}{–}} & \textcolor{gray}{94.82} & \textcolor{gray}{96.00}  \\
    \multicolumn{1}{l|}{\textcolor{gray}{PixMix \cite{Hendrycks_2022_CVPR} + RotPred \cite{hendrycks2019using}}}  & \textcolor{gray}{–} & \multicolumn{1}{c|}{\textcolor{gray}{–}} & \textcolor{gray}{94.86} & \textcolor{gray}{98.18} \\
    \midrule
    \multicolumn{5}{c}{\textbf{Zero-shot methods (no training on ID or extra data)}} \\
    
    \multicolumn{1}{l|}{MCM \cite{ming2022delving}} & 30.86 & \multicolumn{1}{c|}{17.99} & 91.92 & 95.54\\
    \multicolumn{1}{l|}{NegLabel \cite{jiang2024negative}} & 28.75 & \multicolumn{1}{c|}{6.60} & 94.58 & 98.39  \\
    \multicolumn{1}{l|}{AdaNeg \cite{zhang2024adaneg}}  & 20.40
    & \multicolumn{1}{c|}{2.79} & 94.78 & 99.26 \\
    \multicolumn{1}{l|}{InterNeg}  & 23.93 & \multicolumn{1}{c|}{2.59} & 95.13 & 99.29 \\
    \multicolumn{1}{l|}{\cellcolor{top1}\textbf{Ours}}  & \cellcolor{top1}\textbf{17.73} & \multicolumn{1}{c|}{\cellcolor{top1}\textbf{1.47}} & \cellcolor{top1}\textbf{95.53} & \cellcolor{top1}\textbf{99.53} \\
    \bottomrule
    \end{tabular}
    \end{table}

    \begin{table}[!h]
    \small
    \centering
    \caption{Full results of our method with ID dataset of CIFAR-10 on the OpenOOD benchmark.}
    \label{tab:full_cifar10}
    \begin{tabular}{l|l|c|c}
    \toprule
    Near / Far OOD & Datasets & FPR95 $\downarrow$ & AUROC $\uparrow$ \\
    \midrule
    \multirow{3}{*}{Near-OOD} & CIFAR-100 \cite{krizhevsky2009learning} & 30.77 & 92.34  \\
     & TIN \cite{le2015tiny} & 4.68 & 98.72 \\
     & \textbf{Mean} & \textbf{17.73} & \textbf{95.53} \\
    \midrule
    \multirow{5}{*}{Far-OOD} & MNIST \cite{deng2012mnist} & 0.0 & 100.0 \\
     & SVHN \cite{netzer2011reading} & 0.0 & 100.0 \\
     & Texture \cite{cimpoi2014describing} & 0.76 & 99.77 \\
     & Places365 \cite{zhou2017places} & 5.11 & 98.33 \\
     & \textbf{Mean} & \textbf{1.47} & \textbf{99.53} \\
    \bottomrule
    \end{tabular}
    \end{table}

\subsection{Ablation of Different CLIP Architectures}
\label{sp:clip}
Table~\ref{tab:backbone} reports results with different CLIP backbone architectures. 
Our method consistently outperforms prior negative-label-based and test-time adaptation baselines across both CNN-based and ViT-based CLIP encoders. 
These results indicate that our proposed dynamic negative embedding expansion is not tied to a specific CLIP architecture and can be applied to different visual encoders.
% As shown in Table \ref{tab:backbone}, we evaluate multiple CLIP backbone architectures in the Four-OOD setting using ImageNet-1K as the ID dataset. The experimental results demonstrate that our method consistently outperforms all baseline approaches by a substantial margin across different backbone architectures. This robust performance advantage highlights both the effectiveness and robustness of our proposed method.

\begin{table}[!ht]
\centering
\caption{OOD detection results with ID dataset of ImageNet-1K and traditional Four-OOD datasets using different CLIP backbone architectures.}
\label{tab:backbone}
\begin{adjustbox}{width=\textwidth}
\begin{tabular}{lccccccccccc}
\toprule
\multicolumn{1}{l|}{\multirow{3}{*}{Backbones}} & 
\multicolumn{1}{c|}{\multirow{3}{*}{Methods}} & \multicolumn{8}{c|}{OOD Datasets}  & \multicolumn{2}{c}{\multirow{2}{*}{Average}} \\
\multicolumn{1}{l|}{} & \multicolumn{1}{l|}{}  & \multicolumn{2}{c}{iNaturalist}     & \multicolumn{2}{c}{SUN}             & \multicolumn{2}{c}{Places}          & \multicolumn{2}{c|}{Textures}                            & \multicolumn{2}{c}{}                         \\ \cmidrule{3-12} 
\multicolumn{1}{l|}{} & \multicolumn{1}{l|}{} & AUROC$\uparrow$ & FPR95$\downarrow$ & AUROC$\uparrow$ & FPR95$\downarrow$ & AUROC$\uparrow$ & FPR95$\downarrow$ & AUROC$\uparrow$ & \multicolumn{1}{c|}{FPR95$\downarrow$} & AUROC$\uparrow$      & FPR95$\downarrow$     \\ \midrule

\multicolumn{1}{l|}{\multirow{5}{*}{ResNet50}} & \multicolumn{1}{c|}{NegLabel} & 99.24 & 2.88 & 94.54 & 26.51 & 89.72 & 42.60 & 88.40 & \multicolumn{1}{c|}{50.80} & 92.97 & 30.70 \\
 \multicolumn{1}{c|}{} & 
 \multicolumn{1}{c|}{CSP} & 99.46 & 1.95 & 95.73 & 19.05 & 90.39 & 38.58 & 92.41 & \multicolumn{1}{c|}{32.66} & 94.50 & 23.06 \\
  \multicolumn{1}{c|}{} & 
 \multicolumn{1}{c|}{AdaNeg} & 99.58 & 1.18 & 97.37 & 10.56 & 93.84 & 43.19 & 94.18 & \multicolumn{1}{c|}{35.00} & 96.24 & 22.48 \\
  \multicolumn{1}{c|}{} & 
 \multicolumn{1}{c|}{InterNeg} & 99.56 & 1.16 & 98.35 & 7.99 & 93.71 & 37.82 & 96.05 &   \multicolumn{1}{c|}{21.92} & 96.92 & 17.22 \\
 \multicolumn{1}{c|}{} & 
 \multicolumn{1}{c|}{Ours} & \textbf{99.85} & \textbf{0.43} & \textbf{99.05} & \textbf{4.46} & \textbf{96.19} & \textbf{15.87} & \textbf{97.55} & \multicolumn{1}{c|}{\textbf{11.42}} & \textbf{98.16} & \textbf{8.04} \\
\midrule

\multicolumn{1}{l|}{\multirow{5}{*}{ResNet101}} & \multicolumn{1}{c|}{NegLabel} & 99.27 & 3.11 & 94.96 & 24.55 & 89.42 & 44.82 & 87.22 & \multicolumn{1}{c|}{52.78} & 92.72 & 31.32 \\
 \multicolumn{1}{c|}{} & 
 \multicolumn{1}{c|}{CSP} & 99.47 & 2.04 & 95.71 & 19.50 & 90.27 & 39.57 & 90.59 & \multicolumn{1}{c|}{38.67} & 94.01 & 24.95 \\  \multicolumn{1}{c|}{} &
  \multicolumn{1}{c|}{AdaNeg} & 99.69 & 0.78 & 97.65 & 10.61 & 94.00 & 40.38 & 93.59 & \multicolumn{1}{c|}{39.44} & 96.23 & 22.80 \\
  \multicolumn{1}{c|}{} & 
 \multicolumn{1}{c|}{InterNeg} & 99.64 & 0.87 & 98.54 & 7.02 & 93.70 & 38.32 & 95.94 & \multicolumn{1}{c|}{24.74} & 96.96 & 17.74 \\
 \multicolumn{1}{c|}{} & 
 \multicolumn{1}{c|}{Ours} & \textbf{99.89} & \textbf{0.29} & \textbf{99.13} & \textbf{3.88} & \textbf{96.87} & \textbf{13.35} & \textbf{98.00} & \multicolumn{1}{c|}{\textbf{9.43}} & \textbf{98.47} & \textbf{6.74} \\
\midrule

\multicolumn{1}{l|}{\multirow{5}{*}{ViT-B/32}} & \multicolumn{1}{c|}{NegLabel} & 99.11 & 3.73 & 95.27 & 22.48 & 91.72 & 34.94 & 88.57 &  \multicolumn{1}{c|}{50.51} & 93.67 & 27.92 \\
 \multicolumn{1}{c|}{} & 
 \multicolumn{1}{c|}{CSP} & 99.46 & 2.37 & 96.49 & 15.01 & 92.42 & 31.47 & 93.64 &  \multicolumn{1}{c|}{25.09} & 95.50 & 18.49 \\
   \multicolumn{1}{c|}{} & 
  \multicolumn{1}{c|}{AdaNeg} & 99.67 & 0.87 & 97.74 & 9.62 & 93.98 & 36.45 & 94.58 & \multicolumn{1}{c|}{33.26} & 96.49 & 20.05 \\
  \multicolumn{1}{c|}{} & 
 \multicolumn{1}{c|}{InterNeg} & 99.68 & 0.70 & 98.74 & 5.79 & 93.65 & 38.05 & 96.02 & \multicolumn{1}{c|}{23.55} & 97.02 & 17.02 \\
 \multicolumn{1}{c|}{} & 
 \multicolumn{1}{c|}{\textbf{Ours}} & \textbf{99.90} & \textbf{0.28} & \textbf{99.09} & \textbf{4.11} & \textbf{97.16} & \textbf{12.68} & \textbf{97.67} & \multicolumn{1}{c|}{\textbf{10.74}} & \textbf{98.46} & \textbf{6.95} \\
\midrule

\multicolumn{1}{l|}{\multirow{5}{*}{ViT-B/16}} & \multicolumn{1}{c|}{NegLabel} & 99.49 & 1.91 & 95.49 & 20.53 & 91.64 & 35.59 & 90.22 & \multicolumn{1}{c|}{43.56} & 94.21 & 25.40 \\
\multicolumn{1}{c|}{} & 
 \multicolumn{1}{c|}{CSP}  & 99.61 & 1.54 & 96.69 & 13.82 & 92.85 & 29.69 & 93.78 & \multicolumn{1}{c|}{25.78} & 95.73 & 17.71 \\
   \multicolumn{1}{c|}{} & 
  \multicolumn{1}{c|}{AdaNeg} & 99.71 & 0.59 & 97.44 & 9.50 & 94.55 & 34.34 & 94.93 & \multicolumn{1}{c|}{31.27} & 96.66 & 18.93  \\
  \multicolumn{1}{c|}{} & 
 \multicolumn{1}{c|}{InterNeg} & 99.79 & 0.40 & 98.68 & 6.78 & 95.01 & 27.11 & 96.26 & \multicolumn{1}{c|}{21.85} & 97.43 & 14.04 \\
 \multicolumn{1}{c|}{} & 
 \multicolumn{1}{c|}{\textbf{Ours}} & \textbf{99.93} & \textbf{0.21} & \textbf{99.14} & \textbf{3.84} & \textbf{97.14} & \textbf{12.73} & \textbf{97.83} & \multicolumn{1}{c|}{\textbf{10.09}} & \textbf{98.51} & \textbf{6.72} \\
% \midrule
% \multicolumn{1}{l|}{\multirow{5}{*}{ViT-L/14}} & \multicolumn{1}{c|}{NegLabel}& 99.53 & 1.77 & 95.63 & 22.33 & 93.01 & 32.22 & 89.71 & \multicolumn{1}{c|}{42.92} & 94.47 & 24.81 \\
% \multicolumn{1}{c|}{} & 
%  \multicolumn{1}{c|}{CSP}  & 99.72 & 1.21 & 96.73 & 14.88 & 93.58 & 28.41 & 92.71 & \multicolumn{1}{c|}{28.16} & 95.69 & 18.17 \\
%     \multicolumn{1}{c|}{} & 
%   \multicolumn{1}{c|}{AdaNeg} & 99.82 & 0.26 & 97.97 & 7.94 & 95.12 & 28.67 & 94.24 & \multicolumn{1}{c|}{38.28} & 96.79 & 18.79\\
%   \multicolumn{1}{c|}{} & 
%  \multicolumn{1}{c|}{InterNeg} & 99.88 & 0.21 & 98.75 & 5.88 & 95.03 & 26.82 & 96.07 & \multicolumn{1}{c|}{22.89} & 97.43 & 13.95 \\
%  \multicolumn{1}{c|}{} & 
%  \multicolumn{1}{c|}{\textbf{Ours}} & \textbf{99.96} & \textbf{0.09} & \textbf{99.10} & \textbf{4.51} & \textbf{96.67} & \textbf{14.84} & \textbf{96.79} & \multicolumn{1}{c|}{\textbf{13.49}} & \textbf{98.13} & \textbf{8.23} \\
\bottomrule
\end{tabular}
\end{adjustbox}
\end{table}

\subsection{Detailed Results of Temporal Shifts}
Table~\ref{tab:temporal_shifts_vitb16} provides the detailed results under the Temporal-shift setting. These results show that our buffer-based update strategy achieves significant performance improvements under various temporal shifts.

\begin{table}[h]
    \centering
    \caption{OOD detection results under the Temporal-shift setting, where ImageNet-1K ID dataset and a ViTB/16 CLIP encoder are adopted.}
    \label{tab:temporal_shifts_vitb16}
    \resizebox{\textwidth}{!}{
    \begin{tabular}{lcccccccccc}
        \toprule
        \multirow{2}{*}{Methods} & \multicolumn{2}{c|}{iNaturalist} & \multicolumn{2}{c|}{Sun} & \multicolumn{2}{c|}{Places} & \multicolumn{2}{c|}{Textures} & \multicolumn{2}{c}{Average} \\
        & AUROC$\uparrow$ & \multicolumn{1}{c|}{FPR95$\downarrow$} & AUROC$\uparrow$ & \multicolumn{1}{c|}{FPR95$\downarrow$} & AUROC$\uparrow$ & \multicolumn{1}{c|}{FPR95$\downarrow$} & AUROC$\uparrow$ & \multicolumn{1}{c|}{FPR95$\downarrow$} & AUROC$\uparrow$ & FPR95$\downarrow$ \\
        \midrule
        NegLabel & 99.49 & 1.91 & 99.49 & 20.53 & 91.64 & 35.59 & 90.22 & 43.56 & 94.21 & 25.40 \\
        \midrule

        %% Block 1
        \multicolumn{11}{c}{\makebox[0.3em][c]{\textit{i}}~$\rightarrow$~\makebox[0.3em][c]{\textit{S}}~$\rightarrow$~\makebox[0.3em][c]{\textit{P}}~$\rightarrow$~\makebox[0.3em][c]{\textit{T}}} \\

        w/o buffer & 99.93 & 0.17 & 98.70 & 6.38 & 96.81 & 13.68 & 92.45 & 37.18 & 96.97 & 14.35 \\
        w/ buffer  & 99.93 & 0.21 & 98.89 & 5.32 & 97.10 & 12.29 & 95.39 & 22.84 & \textbf{97.83} & \textbf{10.17} \\
        \midrule

        %% Block 2
        \multicolumn{11}{c}{\makebox[0.3em][c]{\textit{S}}~$\rightarrow$~\makebox[0.3em][c]{\textit{P}}~$\rightarrow$~\makebox[0.3em][c]{\textit{T}}~$\rightarrow$~\makebox[0.3em][c]{\textit{i}}} \\
        w/o buffer & 98.71 & 3.49 & 99.14 & 4.04 & 96.64 & 14.43 & 92.45 & 37.73 & 96.74 & 14.92 \\
        w/ buffer  & 99.73 & 0.91 & 99.14 & 3.84 & 97.01 & 12.64 & 95.01 & 23.71 & \textbf{97.72} & \textbf{10.28} \\
        \midrule

        %% Block 3
        \multicolumn{11}{c}{\makebox[0.3em][c]{\textit{P}}~$\rightarrow$~\makebox[0.3em][c]{\textit{T}}~$\rightarrow$~\makebox[0.3em][c]{\textit{i}}~$\rightarrow$~\makebox[0.3em][c]{\textit{S}}} \\
        w/o buffer & 98.87 & 2.90 & 98.20 & 7.17 & 97.13 & 12.50 & 93.39 & 33.21 & 96.90 & 13.95 \\
        w/ buffer  & 99.75 & 0.89 & 98.93 & 5.24 & 97.14 & 12.73 & 95.07 & 22.75 & \textbf{97.72} & \textbf{10.40} \\
        \midrule

        %% Block 4
        \multicolumn{11}{c}{\makebox[0.3em][c]{\textit{T}}~$\rightarrow$~\makebox[0.3em][c]{\textit{i}}~$\rightarrow$~\makebox[0.3em][c]{\textit{S}}~$\rightarrow$~\makebox[0.3em][c]{\textit{P}}} \\
        w/o buffer & 99.45 & 1.54 & 96.27 & 20.16 & 95.17 & 21.11 & 97.81 & 10.12 & 97.18 & 13.23 \\
        w/ buffer  & 99.80 & 0.72 & 98.62 & 6.51 & 96.87 & 12.82 & 97.83 & 10.09 & \textbf{98.28} & \textbf{7.54} \\
        \bottomrule
    \end{tabular}
    }
\end{table}

\subsection{ID/OOD Data Ordering} \label{sec:data_order}
% \begin{table}[h]
%     \centering
%     \caption{OOD detection results under the temporal shifts, where ImageNet-1K ID dataset and a ViTB/16 CLIP encoder are adopted.}
%     \label{tab:temporal_shifts_vitb16_dup}
%     \resizebox{\textwidth}{!}{
%     \begin{tabular}{lcccccccccc}
%         \toprule
%         \multirow{2}{*}{Methods} & \multicolumn{2}{c|}{iNaturalist} & \multicolumn{2}{c|}{Sun} & \multicolumn{2}{c|}{Places} & \multicolumn{2}{c|}{Textures} & \multicolumn{2}{c}{Average} \\
%         & AUROC$\uparrow$ & \multicolumn{1}{c|}{FPR95$\downarrow$} & AUROC$\uparrow$ & \multicolumn{1}{c|}{FPR95$\downarrow$} & AUROC$\uparrow$ & \multicolumn{1}{c|}{FPR95$\downarrow$} & AUROC$\uparrow$ & \multicolumn{1}{c|}{FPR95$\downarrow$} & AUROC$\uparrow$ & FPR95$\downarrow$ \\
%         \midrule
%         NegLabel & 99.49 & 1.91 & 99.49 & 20.53 & 91.64 & 35.59 & 90.22 & 43.56 & 94.21 & 25.40 \\
%         \midrule
%         \multicolumn{11}{c}{\textbf{InterNeg}} \\
%         Forward & 99.93 & 0.25 & 99.61 & 1.32 & 97.61 & 9.28 & 98.93 & 5.48 & 99.02 & 4.08 \\
%         Reverse & 99.88 & 0.41 & 99.04 & 4.80 & 96.33 & 16.31 & 96.20 & 14.93 & 97.86 & 9.11 \\
%         \midrule
%         \multicolumn{11}{c}{\textbf{Ours}} \\
%         Forward & 99.95 & 0.24 & 99.72 & 1.12 & 97.95 & 7.96 & 99.01 & 4.89 & 99.15 & 3.55 \\
%         Reverse & 99.93 & 0.21 & 99.33 & 3.35 & 97.76 & 10.01 & 98.04 & 8.72 & 98.76 & 5.57 \\
%         \bottomrule
%     \end{tabular}
%     }
%   \end{table}
We further evaluate the robustness of our method under different batch ordering patterns, following the protocol used in CLIP-Scope~\cite{fu2025clipscope}. In addition to the default random ordering, we consider forward and reverse ordering patterns:
\begin{itemize}
	\item Forward: all ID samples are processed before all OOD samples.
	\item Reverse: all OOD samples are processed before all ID samples.
\end{itemize}
As shown in Table~\ref{tab:temporal_shifts_vitb16_dup}, Our method consistently outperforms both the static NegLabel baseline and the dynamic InterNeg baseline under both forward and reverse settings. In the forward order, our method reduces the average FPR95 from 4.08\% to 3.55\%, while in the reverse order, it achieves a larger reduction from 9.11\% to 5.57\%. The consistent gains across different arrival orders indicate that our method is robust to batch ordering variations and can maintain reliable adaptation under changing test-time distributions.
  \begin{table}[h]
    \centering
    \caption{Performance (\%) of the proposed TINS method and other methods under two extreme ID/OOD ordering patterns (Forward, Reverse). The utilized CLIP model is ViT-B/16. The ID dataset is ImageNet-1K.}
    \label{tab:temporal_shifts_vitb16_dup}
    \resizebox{\textwidth}{!}{
    \begin{tabular}{lcccccccccc}
        \toprule
        \multirow{2}{*}{Methods} & \multicolumn{2}{c|}{iNaturalist} & \multicolumn{2}{c|}{Sun} & \multicolumn{2}{c|}{Places} & \multicolumn{2}{c|}{Textures} & \multicolumn{2}{c}{Average} \\
        & AUROC$\uparrow$ & \multicolumn{1}{c|}{FPR95$\downarrow$} & AUROC$\uparrow$ & \multicolumn{1}{c|}{FPR95$\downarrow$} & AUROC$\uparrow$ & \multicolumn{1}{c|}{FPR95$\downarrow$} & AUROC$\uparrow$ & \multicolumn{1}{c|}{FPR95$\downarrow$} & AUROC$\uparrow$ & FPR95$\downarrow$ \\
        \midrule
        NegLabel & 99.49 & 1.91 & 99.49 & 20.53 & 91.64 & 35.59 & 90.22 & 43.56 & 94.21 & 25.40 \\
        \midrule
        \multicolumn{11}{c}{\textbf{Forward}} \\
        InterNeg & 99.93 & 0.25 & 99.61 & 1.32 & 97.61 & 9.28 & 98.93 & 5.48 & 99.02 & 4.08 \\
        \rowcolor{top1} \textbf{Ours} & \textbf{99.95} & \textbf{0.24} & \textbf{99.72} & \textbf{1.12} & \textbf{97.95} & \textbf{7.96} & \textbf{99.01} & \textbf{4.89} & \textbf{99.15} & \textbf{3.55} \\
        \midrule
        \multicolumn{11}{c}{\textbf{Reverse}} \\
        InterNeg & 99.88 & 0.41 & 99.04 & 4.80 & 96.33 & 16.31 & 96.20 & 14.93 & 97.86 & 9.11 \\
        \rowcolor{top1} \textbf{Ours} & \textbf{99.93} & \textbf{0.21} & \textbf{99.33} & \textbf{3.35} & \textbf{97.76} & \textbf{10.01} & \textbf{98.04} & \textbf{8.72} & \textbf{98.76} & \textbf{5.57} \\
        \bottomrule
    \end{tabular}
    }
  \end{table}

  \subsection{Analysis under Different ID:OOD Ratios} \label{sup:ratios}
To evaluate our method under imbalanced ID and OOD test settings~\cite{han2024aucseg}, we construct test sets with varying ID:OOD ratios. Specifically, we keep 10,000 SUN samples as OOD data and sample 200, 1,000, 2,000, 10,000, 50,000, 100,000, and 500,000 images from the ImageNet training set as ID data, forming ID:OOD ratios of 1:50, 1:10, 1:5, 1:1, 5:1, 10:1, and 50:1, respectively. As shown in Table~\ref{tab:id_ood_ratio_fpr95}, our method consistently outperforms the strongest baseline InterNeg, and also surpasses the static baseline NegLabel under every ID:OOD ratio. The gains remain clear in both OOD-dominated and ID-dominated regimes, demonstrating that our method is robust to changes in the test composition. 
\begin{table}[h]
    \centering
    \caption{FPR95 ($\downarrow$) across different ID:OOD ratio. `Default' denotes the default Four-OOD benchmark setting, where no specific ID:OOD ratio is enforced.}
    \resizebox{0.7\linewidth}{!}{
    \begin{tabular}{l|cccccccc}
        \toprule
        ID:OOD Ratio & 1:50  & 1:10 & 1:5 & 1:1 & 5:1 & 10:1 & 50:1 & Default \\
        \midrule
        NegLabel & - & - & - & - & - & - & - & 20.53 \\
        InterNeg & 3.80 & 4.67 & 4.78 & 5.34 & 5.82 & 6.73 & 15.66 & 6.78 \\
        \rowcolor{top1} \textbf{Ours} & \textbf{2.25} & \textbf{2.18} & \textbf{2.13} & \textbf{2.47} & \textbf{2.65} & \textbf{3.48} & \textbf{12.22} & \textbf{3.84} \\
   
        \bottomrule
    \end{tabular}
    }
    \label{tab:id_ood_ratio_fpr95}
\end{table}

% | Ratio | Ours FPR95 | InterNeg FPR95 | Ours AUROC | InterNeg AUROC | Ours AUPR | InterNeg AUPR |
% | --- | ---: | ---: | ---: | ---: | ---: | ---: |
% | 1:50 | 3.25 | 3.80 | 99.50 | 99.17 | 87.73 | 82.34 |
% | 1:10 | 2.18 | 4.67 | 99.42 | 98.80 | 96.14 | 92.98 |
% | 1:5 | 2.13 | 4.78 | 99.52 | 98.97 | 97.91 | 95.90 |
% | 1:1 | 2.47 | 5.34 | 99.38 | 98.90 | 99.40 | 98.93 |
% | 5:1 | 2.65 | 5.82 | 99.33 | 98.80 | 99.86 | 99.75 |
% | 10:1 | 3.48 | 6.73 | 99.23 | 98.65 | 99.92 | 99.86 |
% | 50:1 | 12.22 | 15.66 | 97.74 | 96.98 | 99.95 | 99.93 |

\subsection{Sensitivity to Corpus Choice} \label{sup:corpus}
We also study the impact of using different corpora for constructing negative labels. 
Specifically, the default WordNet corpus is replaced by Common-20K and Part-of-Speech, respectively. 
Table~\ref{tab:corpus} shows that our method remains consistently superior to the strongest competing method under both alternatives. 
This demonstrates that our proposed TINS method is robust to the choice of the underlying corpus.
\begin{table}[!h]
    \centering
    \caption{Evaluation with different corpus sources on the Four-OOD benchmark, using ImageNet-1K as the ID dataset.}
    \begin{tabular}{l|l|cc}
        \toprule
        \multirow{2}{*}{Source} & \multirow{2}{*}{Method} & \multicolumn{2}{c}{Average} \\
        \cmidrule{3-4} 
        & & AUROC $\uparrow$ & FPR95 $\downarrow$ \\
        \midrule

        \multirow{5}{*}{Common-20K} 
        & NegLabel & 90.50 & 43.02 \\
        & CSP & 92.06 & 36.56 \\
        & AdaNeg     & 93.12 & 32.39 \\
        & InterNeg & 94.58 & 28.77 \\
        & \textbf{Ours} & \textbf{98.25}  & \textbf{7.48} \\
        \midrule

        \multirow{5}{*}{Part-of-Speech}
        & NegLabel & 92.71 & 32.12 \\
        & CSP & 94.21 & 24.42 \\
        & AdaNeg     & 95.07 & 23.17 \\
        & InterNeg & 95.98 & 18.41 \\
        & \textbf{Ours} & \textbf{98.52} &  \textbf{6.77} \\
        \bottomrule
    \end{tabular}
    \label{tab:corpus}
\end{table}

% \begin{wrapfigure}{r}{0.45\linewidth}
%     \centering
%     \includegraphics[width=0.98\linewidth]{fig/ablation_bs.pdf}
%     \caption{Ablation study on buffer size ($B$) for the overflow buffer, evaluating its effect on AUROC and FPR95. The results demonstrate that an appropriate buffer size helps maintain diversity in the dynamic negative bank and contributes to more stable and robust OOD detection.}
%     \label{fig:ablation_bs}
% \end{wrapfigure}

%%%%%%%%%%%%%%%%%%%%%%%%%%%%%%%%%%%%%%%%%%%%%%%%%%%%%%%%%%%%
% \clearpage
% \newpage
% \input{checklist.tex}

\end{document}